\documentclass[journal]{IEEE}

\usepackage[colorlinks,urlcolor=blue,linkcolor=blue,citecolor=blue]{hyperref}

\usepackage{color,array}

\usepackage{graphicx}
\usepackage{amsmath}
\usepackage{multirow}
\usepackage{booktabs}
\usepackage{amsfonts} 
\usepackage{algorithm}
\usepackage{algorithmicx}
\usepackage{algpseudocode}
 \usepackage{caption}
\usepackage{subcaption}
\usepackage{bm}
\usepackage{booktabs}
\usepackage{orcidlink}
\usepackage{xcolor}
\usepackage{amssymb}

\hypersetup{draft}
\usepackage{titletoc}

\newcommand{\R}{\mathbb{R}}
\newcommand{\vx}{\mathbf{x}}
\newcommand{\vy}{\mathbf{y}}

\newcommand{\vw}{\mathbf{w}}
\newcommand{\vtheta}{{\boldsymbol{\theta}}}
\newcommand{\calS}{\mathcal{S}}

\DeclareMathOperator*{\argmin}{arg\,min}

\newcommand{\diam}{\mathrm{diam}}

\begin{document}

\title{Hinge Regression Trees and HRT-Boost: Newton-Optimized Oblique Learning for Compact Tabular Models}

\author{Hongyi Li \orcidlink{0000-0003-3585-7365}, Jun Xu \orcidlink{0000-0002-2934-4814},~\IEEEmembership{Senior Member,~IEEE}, Hong Yan  \orcidlink{0000-0001-9661-3095},~\IEEEmembership{Life Fellow,~IEEE}

\thanks{A preliminary version of this work was presented at the International Conference on Learning Representations (ICLR) 2026 \cite{li2026hinge}. This manuscript substantially extends the conference version by introducing HRT-Boost, a synergistic ensemble extension, along with new theoretical analysis on stage-wise risk reduction and comprehensive ensemble-level empirical evaluations.}

\thanks{
Hongyi Li and Jun Xu are with the School of Intelligence Science and Engineering, Harbin Institute of Technology, Shenzhen, 518055, China 
and with the Shenzhen Key Lab for Advanced Motion Control and Modern Automation Equipments, 
Shenzhen, 518055, China
(email: 23b904015@stu.hit.edu.cn; xujunqgy@hit.edu.cn).

Hong Yan is with the Department of Electrical Engineering, City University of Hong Kong, Kowloon, Hong Kong (e-mail: h.yan@cityu.edu.hk)
}

\thanks{This work has been submitted to the IEEE for possible publication. Copyright may be transferred without notice, after which this version may no longer be accessible.}
}

\markboth{IEEE Transactions on xxxx, Vol. 00, No. 0, Month 2026}
{H. Li \MakeLowercase{\textit{et al.}}: Hinge Regression Trees and HRT-Boost}

\maketitle

\newtheorem{lemma}{Lemma}
\newtheorem{remark}{Remark}
\newtheorem{theorem}{Theorem}
\newtheorem{corollary}{Corollary}
\newtheorem{assumption}{Assumption}
\newtheorem{proposition}{Proposition}

\begin{abstract}
Learning high-quality oblique decision trees remains a significant challenge due to the discrete and non-convex nature of split optimization. We present the Hinge Regression Tree (HRT) framework, which reframes each oblique split as a nonlinear least-squares problem over two linear predictors whose max/min envelope induces ReLU-like representation capacity. We show that the resulting node-level optimization can be interpreted as a damped Newton method, and we establish the monotonic decrease of the node objective for its backtracking line-search variant. We establish, theoretically, that HRT is a universal approximator with an explicit $O(\delta^2)$ approximation rate. \textcolor{black}{Building upon this base learner, we propose HRT-Boost, a mathematically synergistic ensemble extension that couples node-level Newton updates with stage-wise functional gradient descent. We show that this ensemble construction admits a stage-wise empirical risk reduction guarantee under the squared loss.} Empirical evaluations on synthetic and real-world benchmarks show that HRT is highly competitive with established single-tree baselines, and HRT-Boost compares favorably with strong ensemble baselines and often yields substantially more compact models.
The code is publicly available at \url{https://github.com/Hongyi-Li-sz/HRT-Boost}.
\end{abstract}


\section{Introduction}

\IEEEPARstart{D}{ecision} trees are among the most influential models in supervised learning due to their interpretability and ability to capture nonlinear relationships. 
The classical CART framework \cite{breiman1984cart} introduced axis-aligned recursive partitioning, which remains a cornerstone of modern tree-based methods. 
However, such axis-aligned trees often require deep structures to approximate even simple relationships in high-dimensional or correlated settings, limiting efficiency and generalization \cite{hastie2009elements}.

To address these issues, oblique regression trees extend splitting criteria from axis-aligned thresholds to hyperplanes defined by linear combinations of features. This formulation yields more compact structures and improved predictive performance, particularly when features are correlated \cite{murthy1994oc1}. Nonetheless, finding the optimal oblique hyperplane is inherently challenging due to the discrete and non-convex nature of the split optimization \cite{laurent1976constructing, hancock1996lower}. Practical algorithms have traditionally relied on greedy heuristics, evolutionary methods, or convex surrogates \cite{loh2014fifty, panda2024vanilla, karthikeyanlearning}, but efficient and theoretically sound solutions that scale well remain limited.

We introduce a novel oblique regression tree framework, termed the Hinge Regression Tree (HRT), that fundamentally redefines the node-splitting problem. Specifically, we propose to learn each node split as a nonlinear least squares optimization involving two distinct linear models. The basis function adopts a hinge formulation, intrinsically endowing HRT with ReLU-like representation capacity. We characterize this iterative optimization as a damped Newton method and prove the monotonic decrease of the objective under a backtracking line search. \textcolor{black}{While a single HRT provides a powerful and structurally efficient predictor, relying on a single tree to capture all complex, multi-scale interactions in tabular data can be limiting. Functional gradient boosting elegantly addresses this by constructing additive models that sequentially correct residual errors. However, rather than treating boosting merely as a generic wrapper, we demonstrate that HRT serves as a mathematically synergistic base learner for this framework.}
\textcolor{black}{
The core advantage of HRT-Boost lies in its intrinsic two-level optimization scheme: at the node level, damped Newton updates are employed to capture complex local data geometries, while at the ensemble level, stage-wise gradient descent is used to globally minimize the empirical risk. Because each HRT base learner is highly expressive and optimized directly for the residual landscape, it provides a more informative descent direction in the function space. This synergy allows HRT-Boost to achieve strong predictive performance while efficiently reducing the empirical risk at each stage, resulting in ensemble models that are structurally parsimonious compared to traditional boosting approaches.}

Our contributions can be summarized as follows:
\begin{enumerate}
    \item \textcolor{black}{We introduce the Hinge Regression Tree (HRT), a novel tree-based modeling framework that reformulates node splitting as a nonlinear least squares problem. By hierarchically composing max/min envelope-based splits, HRT essentially forms a circuit of ReLU-like operations, unlocking powerful representation capacity while maintaining structural parsimony.}
    \item We characterize the node-level alternating optimization as a damped Newton method and establish the monotonic decrease of the node objective for its line-search variant. Under a fixed partition, the iteration converges to the corresponding OLS solution.
    \item We establish that the HRT model class is a universal approximator with an explicit $O(\delta^2)$ approximation rate, linking its approximation error directly to the partition granularity.
    \item \textcolor{black}{We propose HRT-Boost, a boosting extension that uses Newton-optimized HRTs as base learners for stage-wise residual fitting. Under the squared loss, we establish a stage-wise empirical risk reduction property and empirically show that HRT-Boost achieves highly competitive predictive performance with compact ensemble structures.}
\end{enumerate}

\textcolor{black}{\emph{Differences from the Conference Version:} 
A preliminary version of this work was presented at ICLR 2026 \cite{li2026hinge}, which focused primarily on the foundational formulation of the Hinge Regression Tree (HRT) as a standalone base learner. In this journal submission, we substantially extend the methodology and theoretical analysis to the ensemble learning paradigm. The key new contributions specific to this manuscript include: (1) the proposal of HRT-Boost, which couples node-level Newton updates with stage-wise functional gradient descent; (2) a formal theoretical proof establishing the stage-wise empirical risk reduction guarantee for HRT-Boost under the squared loss; and (3) an extensively expanded empirical evaluation demonstrating that HRT-Boost matches or exceeds the predictive accuracy of highly optimized ensemble baselines while yielding substantially more compact models.}

\section{Related Work}

\paragraph{Oblique regression trees}  
The quest for optimal oblique splits dates back decades, driven by their ability to form more compact and expressive decision boundaries than axis-aligned trees \cite{murthy1994oc1, loh2014fifty}. Early approaches, such as OC1 \cite{murthy1994oc1}, largely relied on greedy heuristics or statistical tests for feature selection, like GUIDE \cite{loh2009improving}. However, the NP-hard nature of finding optimal oblique hyperplanes \cite{laurent1976constructing} motivated a shift towards more sophisticated optimization-based methodologies. These included alternating optimization strategies, such as TAO \cite{carreira2018alternating}, and gradient-based techniques like DGT \cite{karthikeyanlearning} and DTSemNet \cite{panda2024vanilla}, which treat trees as differentiable architectures. While these methods significantly advanced the field, they often relied on specific neural network approximations or complex search heuristics.
Our work formulates the core splitting problem as a direct, nonlinear least squares optimization with a clear equivalence to a damped Newton method, providing a theoretically grounded and efficient alternative that serves as a robust base learner for ensembles.

\paragraph{Ensemble learning and gradient boosting} 
Gradient boosting algorithms, notably XGBoost \cite{chen2016xgboost} and LightGBM \cite{ke2017lightgbm}, have become highly competitive methods for supervised learning on tabular data. These methods build on foundational frameworks that combine multiple weak learners—such as the early boosting algorithms \cite{freund1997decision} and the random subspace method \cite{ho1998random}—to form strong predictors. Gradient boosting was originally formulated as a functional gradient descent approach that constructs additive models by sequentially fitting base learners to the negative gradient of a loss function \cite{friedman2001greedy}. 
Explicit optimization of the underlying tree structures—such as through regularized greedy formulations \cite{johnson2014learning}—has yielded more compact and accurate models than treating base learners as black boxes. 
Alongside tree-based ensembles, deep learning approaches have also been actively adapted for tabular data. Representative methods include TabNet \cite{arik2021tabnet}, which utilizes sequential attention for interpretable feature selection, and TabM \cite{gorishniy2025tabm}, which introduces a parameter-efficient MLP ensemble.

\paragraph{Piecewise linear models and hinge functions}  
Decision trees are essentially piecewise constant or linear approximations that recursively partition the input space. The universal approximation property of piecewise linear functions is a cornerstone of approximation theory \cite{cybenko1989approximation, hornik1991approximation}. Recent developments have confirmed explicit convergence rates for oblique trees \cite{cattaneo2024convergence}. Hinge functions, initially utilized in hinging hyperplanes \cite{breiman1993hinging}, provide piecewise linear primitives. Their use extends from the hinge loss in SVMs \cite{cortes1995support} to ReLU activations in modern neural networks \cite{nair2010rectified, glorot2011deep}. Our method directly leverages hinge formulations to optimize local node splits, thereby improving structural clarity while maintaining ReLU-like representation capacity.

\section{Methodology: Construction of a Hinge Regression Tree}
\label{sec:methodology}

This section describes the construction of a HRT. We first define the global construction objective, then reformulate the oblique split optimization as a nonlinear least-squares problem. We demonstrate that our iterative splitting procedure is equivalent to a damped Newton method, providing a foundation, in principle, for the tree's local and global optimization. Fig.~\ref{fig:hrt_overview} illustrates the proposed HRT framework and its boosting extension, HRT-Boost.

\begin{figure*}[htbp]
\vspace{-0.1in}
\begin{center}
\centerline{\includegraphics[width=0.99\textwidth]{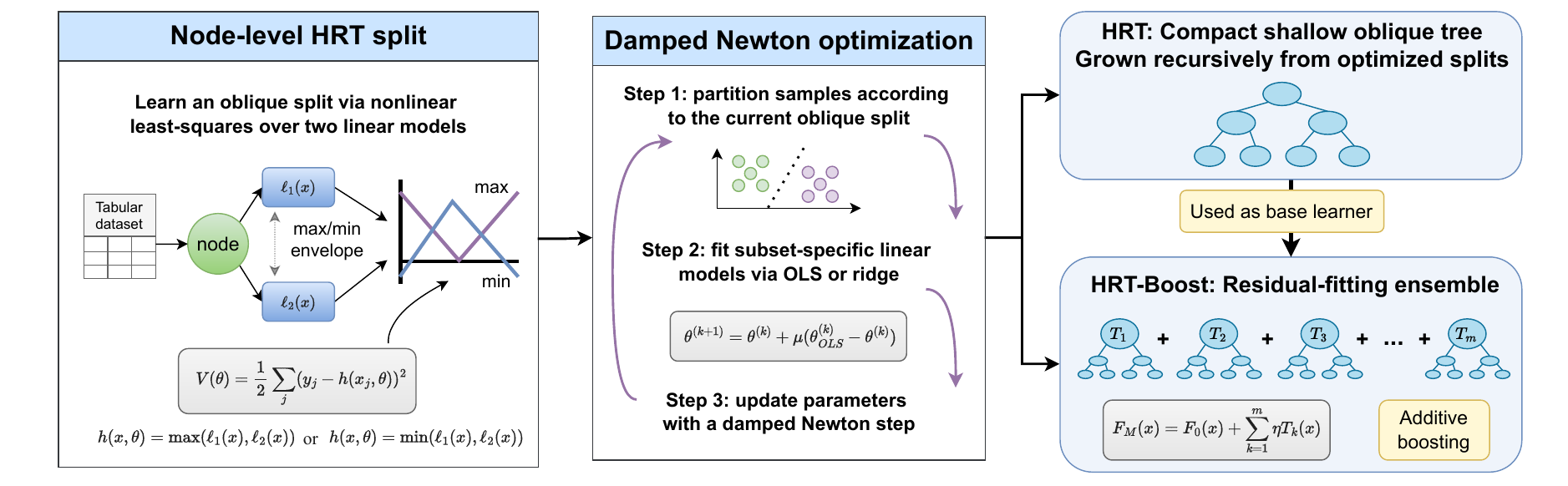}}
\caption{Overview of the proposed HRT framework and its boosting extension, HRT-Boost. 
The node-level HRT split formulates each decision node as a nonlinear least-squares problem over two linear predictors. 
The split parameters are optimized through a damped Newton procedure that alternates between sample partitioning and subset-specific least-squares fitting. 
The resulting compact shallow oblique tree (HRT) can serve as the base learner in the residual-fitting ensemble HRT-Boost.}
\label{fig:hrt_overview}
\end{center}
\vskip -0.2in
\end{figure*}

\subsection{Global Construction Objective}
\label{sec:construction}

Consider a dataset \(\mathcal{S} = \{(\vx_i, y_i)\}_{i=1}^N
\), where \(\vx_i \in \mathbb{R}^{d}\) and \(y_i \in \mathbb{R}\). Let \(\tilde{\vx} \in \R^{d+1}\) be the augmented feature vector, i.e., \(\tilde{\vx} = [x_{1}, \dots, x_{d}, 1]^T\). The tree consists of internal nodes \(t \in \mathbb{T}_{\mathrm{I}}\) for branching tests and leaf nodes \(t \in \mathbb{T}_{\mathrm{L}}\) for prediction. For each node \(\mathbb{D}_t\), there is an associated parameter vector \(\vtheta_t \in \mathbb{R}^{d+1}\). Let \(t_{l}(\vx)\) be the leaf index reached by input \(\vx\), and its linear predictor be \(\ell_{t_l}(\vx)=\tilde{\vx}^T\vtheta_{t_l(\vx)}\). 
The global objective is to minimize the total sum of squared errors over the entire tree:
\begin{equation}\label{eq:global_objective}
   J(\Theta) = \frac{1}{2}\sum_{i=1}^N \bigl(y_i-\hat{y}(\vx_i)\bigr)^2,
\end{equation}
where $\Theta$ collects the parameters of all leaf predictors and $\hat{y}(\vx_i)=\ell_{t_l(\vx_i)}(\vx_i)$.

\subsection{Node Splitting as Nonlinear Least Squares}
\label{sec:local_optimization}

At any internal node \(\mathbb{D}_t\), our goal is to find two sets of parameters \(\vtheta_{t_1}, \vtheta_{t_2} \in \R^{d+1}\) that define two linear functions \(\ell_{t_1}(\vx)\) and \(\ell_{t_2}(\vx)\). We minimize the following nonlinear least squares objective:
\begin{equation}
V({\vtheta}) = \frac{1}{2} \sum_{\vx_j\in \mathbb{D}_t} \left( y_j - h(\vx_j, {\vtheta}) \right)^2,
\label{eq:objective_en}
\end{equation}
with respect to \({\vtheta} = [\vtheta_{t_1}^T, \vtheta_{t_2}^T]^T\). The basis function \(h(\vx_j, {\vtheta})\) is a hinge function defined as:
\[
h(\vx_j, {\vtheta}) = \max \left( \tilde{\vx}_j^T \vtheta_{t_1}, \tilde{\vx}_j^T \vtheta_{t_2} \right) \quad \text{or}  \quad \min \left( \tilde{\vx}_j^T \vtheta_{t_1}, \tilde{\vx}_j^T \vtheta_{t_2} \right).
\]
This hinge formulation intrinsically defines an oblique decision boundary \(\tilde{\vx}^T (\vtheta_{t_1} - \vtheta_{t_2}) = 0\). Depending on which side of this hyperplane a point \(\vx_j\) lies, the model selects either \(\ell_{t_1}\) or \(\ell_{t_2}\). Minimizing \(V(\vtheta)\) thus optimizes a structure with one root and two leaves, allowing the split to adapt to local convex or concave data structures.

\subsection{Optimization via a Damped Newton Method}
\label{newton}

\textcolor{black}{Directly minimizing \eqref{eq:objective_en} is non-trivial due to the non-differentiable hinge. However, by introducing dynamic partitions based on the current parameters \(\vtheta\):}
\begin{align*}
    \calS_1(\vtheta) &= \{\vx_j \in \mathbb{D}_t \mid \tilde{\vx}_j^T \vtheta_{t_1} \ge \tilde{\vx}_j^T \vtheta_{t_2} \}, \quad \calS_2(\vtheta) = \mathbb{D}_t \setminus \calS_1(\vtheta),
\end{align*}
\textcolor{black}{the objective becomes differentiable within each partition. To derive an efficient update, we characterize the optimization as a Newton method. For a fixed partition, the gradient \(\nabla V\) and the Hessian \(\nabla^2 V\) are:}
\begin{equation}
\begin{split}
\nabla V(\vtheta) &= - \begin{pmatrix}
\sum_{\vx_j \in \calS_1} \tilde{\vx}_j (y_j - \tilde{\vx}_j^T \vtheta_{t_1}) \\
\sum_{\vx_j \in \calS_2} \tilde{\vx}_j (y_j - \tilde{\vx}_j^T \vtheta_{t_2})
\end{pmatrix}, \\[1ex]
\nabla^2 V(\vtheta) &= \begin{pmatrix}
\sum_{\vx_j \in \calS_1} \tilde{\vx}_j \tilde{\vx}_j^T & \mathbf{0} \\
\mathbf{0} & \sum_{\vx_j \in \calS_2} \tilde{\vx}_j \tilde{\vx}_j^T
\end{pmatrix}.
\end{split}
\label{eq:grad_hess_full}
\end{equation}
\textcolor{black}{Note that because \(h(\vx_j, \vtheta)\) is locally linear, the Gauss-Newton approximation of the Hessian is exact. The standard Newton update \(\vtheta^{(k+1)} = \vtheta^{(k)} - \mu [\nabla^2 V]^{-1} \nabla V\) simplifies to:}
\begin{equation}
\vtheta^{(k+1)} = \vtheta^{(k)} + \mu (\vtheta_{\text{OLS}}^{(k)} - \vtheta^{(k)}),
\label{eq:newton_update_en}
\end{equation}
\textcolor{black}{where \(\vtheta_{\text{OLS}}^{(k)}\) is the Ordinary Least Squares solution independently computed for each subset \(\calS_i\). 
This reformulation also explains the computational efficiency of HRT: under fixed partitions, the hinge-based split decomposes into two linear least-squares subproblems, so the Newton/Gauss–Newton direction can be obtained in closed form from the corresponding OLS solutions, avoiding a generic second-order optimization step.
When the damping factor \(\mu=1\), the update reduces to \(\vtheta^{(k+1)} = \vtheta_{\text{OLS}}^{(k)}\), meaning the next iteration's parameters are simply the OLS fits of the current partition. In practice, we employ both fixed damping and backtracking line search (\emph{auto}) to ensure stable convergence, especially in non-convex landscapes (see Appendix~A for the full algebraic proof of this equivalence).}

\subsection{Recursive Construction and Robustness}
\label{sec:tree_construction}

The construction follows a recursive routine. Starting from the root, each internal node applies the Newton-based splitting optimization. Data points are assigned to child nodes based on the hinge test, and the process repeats until a stopping criterion (e.g., max depth) is met. 

To enhance robustness against multicollinearity, ridge regression (L2 regularization) can, optionally, be incorporated into the OLS fitting steps. If the iterative optimization at a node fails to converge within a fixed budget, a fallback mechanism ensures consistent tree growth. 
Detailed formulations of these mechanisms, including the ridge-regularized fitting and the fallback strategy, are provided in Appendix~B. The complete pseudocode and complexity analysis are deferred to Appendices~C and D, respectively.

\subsection{Hierarchical ReLU-like Representation Capacity}
\label{sec:relu_logic}

The HRT's representation capacity stems from its hierarchical composition of hinge operations. Each split routes data via a hinge test \(s_t(\vx)=\ell_{t_1}(\vx)-\ell_{t_2}(\vx) \ge 0\). Using the identity \(\max(a,b)=a+\max(0, b-a)\), a node split is essentially a single-unit ReLU gate. A tree of depth \(k\) thus generates a computational circuit of nested ReLU operations, allowing the model to approximate complex nonlinearities with a significantly more compact structure than axis-aligned trees. This compositional view will be useful to understand the boosting extension developed next. While a single HRT realizes a hierarchical hinge composition, HRT-Boost enhances the model class through an additive expansion of these hierarchical modules.

\subsection{HRT-Boost: Functional Gradient Boosting with Hinge Regression Trees}
\label{sec:hrt_boost}

\textcolor{black}{The preceding subsections established that a single HRT forms a compact, piecewise linear predictor through hierarchical hinge compositions, and that each node split can be optimized via a damped Newton procedure. Building on this foundation, we extend HRT from a standalone predictor to an additive ensemble model. Specifically, we propose \emph{HRT-Boost}, which uses the HRT as the base learner in a stage-wise functional gradient boosting framework under the squared loss.}

\textcolor{black}{Given the training set $\mathcal{S}=\{(\vx_i,y_i)\}_{i=1}^N$, we define the ensemble predictor after $m$ boosting stages as
\begin{equation}
F_m(\vx) = F_0(\vx) + \sum_{k=1}^{m} \eta\, T_k(\vx),
\label{eq:hrtboost_ensemble}
\end{equation}
where $T_k$ denotes the $k$-th HRT base learner, and $\eta \in (0,1]$ is the learning rate. The initialization is chosen as the constant predictor
\begin{equation}
F_0(\vx)=\frac{1}{N}\sum_{i=1}^N y_i,
\label{eq:hrtboost_init}
\end{equation}
which minimizes the squared loss over all constants.}

\textcolor{black}{To formalize the boosting procedure, consider the empirical risk
\begin{equation}
\mathcal{L}(F) = \frac{1}{2}\sum_{i=1}^N \bigl(y_i - F(\vx_i)\bigr)^2.
\label{eq:hrtboost_loss}
\end{equation}
Under the squared loss, the negative functional gradient evaluated at the current ensemble $F_{m-1}$ is given by
\begin{equation}
-\left.\frac{\partial \mathcal{L}(F)}{\partial F(\vx_i)}\right|_{F=F_{m-1}}
= y_i - F_{m-1}(\vx_i).
\label{eq:hrtboost_neggrad}
\end{equation}
Hence, the pseudo-residual at stage $m$ coincides exactly with the ordinary residual:
\begin{equation}
r_{i,m} = y_i - F_{m-1}(\vx_i), \qquad i=1,\dots,N.
\label{eq:hrtboost_residual}
\end{equation}}

\textcolor{black}{At the $m$-th stage, we construct an HRT $T_m$ by fitting the residual dataset
\[
\mathcal{S}^{(m)}_{\mathrm{res}}
=
\{(\vx_i,r_{i,m})\}_{i=1}^N.
\]
Equivalently, $T_m$ is obtained by solving
\begin{equation}
T_m
=
\arg\min_{T\in\mathcal{H}}
\frac{1}{2}\sum_{i=1}^N \bigl(r_{i,m}-T(\vx_i)\bigr)^2,
\label{eq:hrtboost_baselearner}
\end{equation}
where $\mathcal{H}$ denotes the prescribed family of HRTs, determined by the chosen depth limit and regularization settings. The ensemble is then updated in the standard stage-wise manner:
\begin{equation}
F_m(\vx) = F_{m-1}(\vx) + \eta\, T_m(\vx).
\label{eq:hrtboost_update}
\end{equation}}

\textcolor{black}{The construction of each $T_m$ is \emph{identical} to that of a standalone HRT, except that the target values are replaced by the residuals $\{r_{i,m}\}_{i=1}^N$. In particular, every internal node still solves the nonlinear least-squares split objective introduced in Section~\ref{sec:local_optimization}, and the corresponding split parameters are optimized via the same damped Newton procedure developed in Section~\ref{newton}. By utilizing the automatic step size (\emph{auto}, via backtracking line search) during these Newton updates, the base learners dynamically adapt to the shrinking scales and shifting landscapes of the pseudo-residuals at each boosting stage, ensuring robust convergence without requiring stage-specific manual tuning. Therefore, HRT-Boost fully preserves the structural properties of HRT---including oblique partitioning, subset-specific linear prediction, and hierarchical hinge composition---while progressively correcting the residual error in an additive manner.}

\textcolor{black}{This construction has a natural interpretation as a two-level optimization. At the \emph{node level}, each split of an HRT base learner is obtained through the damped Newton update derived earlier. At the \emph{ensemble level}, the sequence $\{F_m\}_{m=1}^M$ is produced by stage-wise functional gradient descent with respect to the empirical risk \eqref{eq:hrtboost_loss}. In this sense, HRT-Boost is not merely a generic boosting wrapper around HRT, but rather an extension that combines Newton-style local optimization within each tree and gradient-driven global residual correction across boosting stages.}

\textcolor{black}{HRT-Boost complements the discussion in Section~\ref{sec:relu_logic}. A single HRT realizes a hierarchical composition of hinge operations. HRT-Boost forms an additive expansion of these hierarchical hinge modules through \eqref{eq:hrtboost_ensemble}. This additive residual-correction mechanism reduces the approximation bias of a single tree without modifying the underlying HRT architecture.}

\section{Theoretical Analysis}
\label{sec:theory}

This section analyzes the proposed framework at three levels. 
First, the node-level optimization underlying each HRT split is shown to exhibit monotone descent under damped Newton updates. 
Second, the approximation capability of the HRT model class is characterized by an explicit approximation rate and a universal approximation property. 
Finally, extending the theoretical guarantees of the single-tree base learner, the ensemble-level analysis establishes a stage-wise empirical risk reduction for HRT-Boost under the squared loss.

\subsection{Node-Level Optimization and Convergence}
\label{sec:node_convergence}

At the local optimization level for a single internal node, as established in Section~\ref{newton}, fixing the current partition reduces the node objective to a quadratic least-squares problem. The resulting update coincides with a damped Newton step. The following analysis formalizes the descent and convergence properties of this update rule.

Our node-level algorithm uses the generic iteration
\begin{equation}
\vtheta^{(k+1)} = \vtheta^{(k)} + \mu^{(k)} p^{(k)},
\qquad
p^{(k)} = \vtheta_{\mathrm{OLS}}^{(k)} - \vtheta^{(k)},
\label{eq:node_newton_update}
\end{equation}
where \(p^{(k)}\) is the Newton/Gauss--Newton direction under the current partition.
In practice, we instantiate \eqref{eq:node_newton_update} in two ways.
The first uses a fixed damping factor \(\mu^{(k)}\equiv \mu\in(0,1]\) as a hyperparameter.
The second, denoted as \emph{auto} in our experiments, uses a monotone backtracking search that starts from \(\mu^{(k)}=1\) and geometrically decreases the step size until a strict decrease of the node objective is obtained.

To state the result cleanly, we first rewrite the local objective in matrix form.
For a given partition \((\mathcal{S}_1,\mathcal{S}_2)\) at node \(\mathbb{D}_t\), let \(X_i\in\mathbb{R}^{|\mathcal{S}_i|\times(d+1)}\) denote the design matrix whose rows are the augmented feature vectors \(\tilde{\vx}_j^T\) for all \(\vx_j\in\mathcal{S}_i\), and let \(\vy_i\in\mathbb{R}^{|\mathcal{S}_i|}\) be the corresponding target vector, for \(i\in\{1,2\}\).
Under a fixed partition, the node objective becomes
\begin{equation}
V(\vtheta)
=
\frac{1}{2}\|\vy_1 - X_1\vtheta_{t_1}\|_2^2
+
\frac{1}{2}\|\vy_2 - X_2\vtheta_{t_2}\|_2^2.
\label{eq:node_matrix_objective}
\end{equation}
Accordingly,
\begin{equation}
\begin{aligned}
\nabla V(\vtheta)
&=
\begin{bmatrix}
X_1^T(X_1\vtheta_{t_1}-\vy_1)\\
X_2^T(X_2\vtheta_{t_2}-\vy_2)
\end{bmatrix},\\[0.5em]
\nabla^2 V(\vtheta)
&=
\begin{bmatrix}
X_1^T X_1 & \mathbf{0}\\
\mathbf{0} & X_2^T X_2
\end{bmatrix}.
\end{aligned}
\label{eq:node_grad_hess}
\end{equation}
This makes explicit that the OLS solution decouples across the two subsets, and that the Newton direction reduces to \(\vtheta_{\mathrm{OLS}}-\vtheta\).

We impose the following standard regularity conditions.

\begin{assumption}[Regularity]
\label{ass:node_regularity}
For every iteration, with \(\vtheta^{(k)}\) at node \(\mathbb{D}_t\), the following conditions hold:
\begin{enumerate}
    \item[(i)] (\emph{Nondegenerate hinge})
    For every sample \(\vx_j\in\mathbb{D}_t\),
    \[
    \tilde{\vx}_j^T\vtheta_{t_1}^{(k)} \neq \tilde{\vx}_j^T\vtheta_{t_2}^{(k)}.
    \]
    That is, no sample lies exactly on the separating hyperplane at the current iteration.
    
    \item[(ii)] (\emph{Uniform block strong convexity})
    For every partition encountered along the iterations, the corresponding block Gram matrices satisfy
    \[
    m I \preceq X_i^T X_i \preceq L I,
    \qquad i=1,2,
    \]
    for some global constants \(0<m\le L<\infty\).
\end{enumerate}
\end{assumption}

Assumption~\ref{ass:node_regularity} is standard for piecewise least-squares models.
Condition (i) ensures local smoothness of the hinge objective in the current iteration, while condition (ii) guarantees that the quadratic subproblems remain uniformly well conditioned.

We can now state the node-level convergence result.

\begin{theorem}[Node-level convergence]
\label{thm:node_descent}
Under the node-level setup and Assumption~\ref{ass:node_regularity}, consider the iteration
\[
\vtheta^{(k+1)}=\vtheta^{(k)}+\mu^{(k)}p^{(k)},
\qquad
p^{(k)}=\vtheta_{\mathrm{OLS}}^{(k)}-\vtheta^{(k)},
\]
where \(\mu^{(k)}\) is selected by a backtracking line search with contraction factor \(\beta\in(0,1)\) and candidate step sizes \(\{\mu_0\beta^s\}_{s\ge 0}\).
Then the following hold.
\begin{enumerate}
    \item[(a)] (\emph{Per-iteration of  descent and line-search termination})  
    For any iterative step $\vtheta^{(k)}$ with $\vtheta^{(k)}\neq\vtheta_{\mathrm{OLS}}^{(k)}$, the Newton direction $p^{(k)}$ is a strict descent direction and the backtracking line search produces a step size $\mu^{(k)}>0$ such that
    \begin{equation}\label{eq:node_descent}
        V(\vtheta^{(k+1)}) = V(\vtheta^{(k)} + \mu^{(k)} p^{(k)}) < V(\vtheta^{(k)}).
    \end{equation}
    \item[(b)] (\emph{Monotone decrease and convergence of the objective})  
    The sequence $\{V(\vtheta^{(k)})\}$ is strictly decreasing whenever $p^{(k)}\neq 0$ and converges to a finite limit $V_\infty$.
    \item[(c)] (\emph{Convergence under a fixed partition})  
    If there exists $K$ such that the partition $(\mathcal{S}_1(\vtheta^{(k)}),\mathcal{S}_2(\vtheta^{(k)}))$ is constant for all $k\ge K$, then $\{\vtheta^{(k)}\}$ converges to the unique OLS minimizer for that fixed partition.
\end{enumerate}
\end{theorem}

\begin{proof}
\textcolor{black}{We proceed in several steps.}

\textbf{Step 1: The Newton direction is a strict descent direction.}
In any iteration $\vtheta$ (we temporarily drop the superscript $k$),
the Newton direction satisfies
\[
    p(\vtheta) = -H(\vtheta)^{-1}\nabla V(\vtheta),
\]
so
\[
    \nabla V(\vtheta)^T p(\vtheta)
    = -\nabla V(\vtheta)^T H(\vtheta)^{-1}\nabla V(\vtheta).
\]
Assumption~\ref{ass:node_regularity}(ii) implies
$H(\vtheta)\preceq L I$ and $H(\vtheta)\succ 0$, and hence
$H(\vtheta)^{-1}\succeq \tfrac{1}{L}I$. Therefore,
\[
    \nabla V(\vtheta)^T H(\vtheta)^{-1}\nabla V(\vtheta)
    \ge \frac{1}{L}\|\nabla V(\vtheta)\|_2^2,
\]
which yields
\[
    \nabla V(\vtheta)^T p(\vtheta)
    \le -\frac{1}{L}\|\nabla V(\vtheta)\|_2^2.
\]
If $\vtheta\neq\vtheta_{\mathrm{OLS}}$, then
$\nabla V(\vtheta)\neq 0$, so the right-hand side is strictly negative.
Thus $p(\vtheta)$ is a strict descent direction.

\textbf{Step 2: Local decrease along $p(\vtheta)$.}
Fix $\vtheta$ and its Newton direction $p = p(\vtheta)$.
Write
\[
    \vtheta =
    \begin{bmatrix}
        \vtheta_{t_1} \\[2pt] \vtheta_{t_2}
    \end{bmatrix},
    \qquad
    p(\vtheta) =
    \begin{bmatrix}
        p_1 \\[2pt] p_2
    \end{bmatrix},
\]
to conform with the block structure in Equation~\eqref{eq:grad_hess_full}. 

\textcolor{black}{Consider the univariate function
\[
    \psi(\mu) := V(\vtheta + \mu p), \qquad \mu \ge 0.
\]
We show that there exists $\tilde{\mu}>0$ such that
$V(\vtheta+\mu p) < V(\vtheta)$ for all $0<\mu\le\tilde{\mu}$.}

\textcolor{black}{For each $j$, define
\[
    a_j(\vtheta) = \tilde{\vx}_j^T\vtheta_{t_1},
    \qquad
    b_j(\vtheta) = \tilde{\vx}_j^T\vtheta_{t_2}.
\]
Assumption~\ref{ass:node_regularity}(i) states that
$a_j(\vtheta) \neq b_j(\vtheta)$ for all $j$. Without loss of generality, suppose
$a_j(\vtheta) > b_j(\vtheta)$; the other case is analogous.}

\textcolor{black}{Along the ray $\vtheta + \mu p$ we have
\[
    a_j(\vtheta + \mu p) - b_j(\vtheta + \mu p)
    = \bigl(a_j(\vtheta) - b_j(\vtheta)\bigr)
      + \mu \bigl(\tilde{\vx}_j^T p_1 - \tilde{\vx}_j^T p_2\bigr),
\]
which is an affine function of $\mu$ and is nonzero at $\mu=0$.
Therefore, for each $j$ there exists $\bar{\mu}_j>0$ such that
the sign of $a_j(\vtheta + \mu p) - b_j(\vtheta + \mu p)$ does not change for all $0\le \mu \le \bar{\mu}_j$. In particular, the active branch in $h(\vx_j,\cdot)$ remains the same for all small $\mu$.}

\textcolor{black}{Let $\bar{\mu} := \min_j \bar{\mu}_j > 0$. Then for all $0\le \mu \le \bar{\mu}$ the index of the active branch in $h(\vx_j,\cdot)$
is fixed for every $j$, so $V(\vtheta + \mu p)$ coincides with the smooth quadratic representation \eqref{eq:node_matrix_objective}
along this segment. Hence $\psi(\mu)$ is differentiable on $[0,\bar{\mu}]$
with
\[
    \psi'(0) = \nabla V(\vtheta)^T p.
\]
By Step~1, $\psi'(0) = \nabla V(\vtheta)^T p < 0$.
By continuity of $\psi'$ at $0$, there exists
$0<\tilde{\mu}\le\bar{\mu}$ such that
\[
    \psi(\mu) < \psi(0) = V(\vtheta)
    \quad\text{for all } 0 < \mu \le \tilde{\mu},
\]
i.e.,
\[
    V(\vtheta + \mu p) < V(\vtheta)
    \quad\text{for all } 0 < \mu \le \tilde{\mu}.
\]}

\textbf{Step 3: The backtracking line search terminates through decreasing steps.}
Let the backtracking line search start from some initial step
$\mu_0>0$ and generate the sequence
$\mu_t = \mu_0\beta^t$ for $t=0,1,2,\dots$ with $\beta\in(0,1)$.
Since $\mu_t\to 0$, there exists $t^\star$ such that
$\mu_{t^\star} \le \tilde{\mu}$. For this $t^\star$ we have
$0 < \mu_{t^\star} \le \tilde{\mu}$ and hence, by Step~2,
\[
    V(\vtheta + \mu_{t^\star} p)
    < V(\vtheta).
\]
The backtracking procedure selects the first such $t^\star$, so it
terminates with a step size $\mu^{(k)} = \mu_{t^\star}>0$ satisfying
\eqref{eq:node_descent}. This proves part (a).

\textbf{Step 4: Monotone decrease and convergence of $V$.}
Applying part (a) at each iteration $k$ yields
\[
    V(\vtheta^{(k+1)})
    = V(\vtheta^{(k)} + \mu^{(k)} p^{(k)})
    < V(\vtheta^{(k)}),
\]
whenever $p^{(k)}\neq 0$. Since $V(\vtheta)\ge 0$ for all $\vtheta$,
the sequence $\{V(\vtheta^{(k)})\}$ is strictly decreasing and bounded
below, and thus converges to some finite limit $V_\infty\ge 0$.
This establishes part (b).

\textbf{Step 5: Convergence under a fixed partition.}
Suppose there exists $K$ such that the partition
$(\mathcal{S}_1(\vtheta^{(k)}),\mathcal{S}_2(\vtheta^{(k)}))$
is constant for all $k\ge K$. Then for all $k\ge K$, the objective $V$
coincides with the smooth, strongly convex quadratic version of 
\eqref{eq:node_matrix_objective} with Hessian
\[
    H(\vtheta) = 
    \begin{bmatrix}
        X_1^T X_1 & 0 \\[2pt]
        0 & X_2^T X_2
    \end{bmatrix},
    \qquad
    m I \preceq H(\vtheta) \preceq L I,
\]
by Assumption~\ref{ass:node_regularity}(ii).
In this regime, the Newton direction simplifies to
\[
    p(\vtheta) = \vtheta_{\mathrm{OLS}} - \vtheta,
\]
and the update becomes
\[
    \vtheta^{(k+1)} - \vtheta_{\mathrm{OLS}}
    = \bigl(1 - \mu^{(k)}\bigr)
      \bigl(\vtheta^{(k)} - \vtheta_{\mathrm{OLS}}\bigr).
\]
Under the standard backtracking line-search rule on a strongly convex
quadratic, the accepted step sizes $\mu^{(k)}$ are uniformly bounded
away from zero and from above, so the factor $|1-\mu^{(k)}|$ is strictly
less than $1$ and uniformly bounded away from $1$. Hence
$\vtheta^{(k)} \to \vtheta_{\mathrm{OLS}}$ as $k\to\infty$.
Equivalently, the iteration converges to the unique OLS minimizer for
the fixed partition, which proves part (c) and completes the proof.

\end{proof}

Theorem~\ref{thm:node_descent} formalizes the local optimization behavior of the proposed split procedure.
In particular, the line-search variant guarantees monotone descent of the node objective, while convergence becomes exact as soon as the partition stabilizes.

\subsection{Approximation Rate and Universal Approximation}
\label{sec:approx_theory}

This subsection characterizes the representation capacity of a single HRT.
Specifically, we show that finite oblique regression trees with linear leaves achieve an \(O(\delta^2)\) approximation rate for twice continuously differentiable targets on compact domains.
The universal approximation property then follows as an immediate consequence.

\begin{theorem}[Approximation rate]
\label{thm:approx_rate}
Let \(\mathcal F\) denote the class of piecewise linear functions represented by finite oblique regression trees with linear models at the leaves.
Let \(g:\mathcal K\to\mathbb R\) be twice continuously differentiable on a compact set \(\mathcal K\subset\mathbb R^d\).
Assume that for every constructed partition \(\{\mathcal R_i\}\) of \(\mathcal K\) with \(\max_i \diam(\mathcal R_i)\le \delta\), each region \(\mathcal R_i\) contains \(N_i\) training samples, and the corresponding design matrix \(X_i\) satisfies
\[
\lambda_{\min}(X_i^T X_i)\ge cN_i
\]
for some constant \(c>0\).
Assume further that
\[
\sup_{\vx\in\mathcal K}\|\tilde{\vx}\|_2\le D_{\mathcal K}.
\]
Then there exists a constant \(C>0\), independent of \(\delta\), such that for every \(\delta>0\) one can construct an \(f\in\mathcal F\) satisfying
\begin{equation}
\sup_{\vx\in\mathcal K}|f(\vx)-g(\vx)|
\le
C\delta^2.
\label{eq:approx_rate_bound}
\end{equation}
In particular, \(\mathcal F\) is a universal approximator on \(\mathcal K\).
\end{theorem}

\begin{proof}
We construct the approximation region by region and bound the resulting error in two parts: the local Taylor approximation error and the estimation error of the leafwise linear fit.

\textbf{Step 1: Partition construction.}
Since \(\mathcal K\) is compact, it can be covered by finitely many disjoint convex polytopes \(\{\mathcal R_i\}_{i=1}^{M_\mathcal{R}}\) such that
\[
\bigcup_{i=1}^{M_\mathcal{R}} \mathcal R_i=\mathcal K,
\qquad
\max_i \diam(\mathcal R_i)\le \delta.
\]
Because oblique trees implement recursive hyperplane partitions, such a partition can be realized by a finite oblique regression tree.

\textbf{Step 2: Local first-order approximation.}
For each region \(\mathcal R_i\), choose a reference point \(\mathbf c_i\in\mathcal R_i\) and define the first-order Taylor approximation
\[
f_i^\star(\vx)
=
g(\mathbf c_i)+\nabla g(\mathbf c_i)^T(\vx-\mathbf c_i),
\qquad \vx\in\mathcal R_i.
\]
This is an affine function of \(\vx\), and therefore belongs to the leafwise model class.
In augmented form, it can be written as
\[
f_i^\star(\vx)=\vw_i^{\star T}\tilde{\vx},
\]
for some coefficient vector \(\vw_i^\star\in\mathbb R^{d+1}\).

\textbf{Step 3: Taylor remainder bound.}
Since \(g\in C^2(\mathcal K)\) and \(\mathcal K\) is compact, the Hessian of \(g\) is bounded on \(\mathcal K\).
Let
\[
L_g:=\sup_{\mathbf z\in\mathcal K}\|\nabla^2 g(\mathbf z)\|_2<\infty.
\]
Taylor's theorem then yields, for every \(\vx\in\mathcal R_i\),
\[
|g(\vx)-f_i^\star(\vx)|
\le
\frac{L_g}{2}\|\vx-\mathbf c_i\|_2^2
\le
\frac{L_g}{2}\delta^2.
\]
Define \(C_A:=L_g/2\).
Then
\begin{equation}
\sup_{\vx\in\mathcal R_i}|g(\vx)-f_i^\star(\vx)|
\le
C_A\delta^2.
\label{eq:local_taylor_error}
\end{equation}

\textbf{Step 4: Leafwise estimation error.}
Let \(\hat{\vw}_i\) denote the OLS estimator fitted on region \(\mathcal R_i\).
For the training responses in \(\mathcal R_i\), we have
\[
\vy_i = X_i\vw_i^\star + \mathbf R_i,
\]
where \(\mathbf R_i\) collects the Taylor residuals.
By \eqref{eq:local_taylor_error}, each entry of \(\mathbf R_i\) is bounded in magnitude by \(C_A\delta^2\), and therefore
\[
\|\mathbf R_i\|_2\le \sqrt{N_i}\,C_A\delta^2.
\]
Since
\[
\hat{\vw}_i = (X_i^T X_i)^{-1}X_i^T\vy_i,
\]
we obtain
\[
\hat{\vw}_i-\vw_i^\star
=
(X_i^T X_i)^{-1}X_i^T\mathbf R_i.
\]
Thus, for any \(\vx\in\mathcal R_i\),
\[
\begin{aligned}
|f_i^\star(\vx)-\hat{\vw}_i^T\tilde{\vx}|
&= |\tilde{\vx}^T(\vw_i^\star-\hat{\vw}_i)| \\
&\le \|\tilde{\vx}\|_2\,\|(X_i^T X_i)^{-1}X_i^T\|_2\,\|\mathbf R_i\|_2.
\end{aligned}
\]
Using \(\sup_{\vx\in\mathcal K}\|\tilde{\vx}\|_2\le D_{\mathcal K}\) and
\[
\|(X_i^T X_i)^{-1}X_i^T\|_2
=
\frac{1}{\sigma_{\min}(X_i)}
\le
\frac{1}{\sqrt{cN_i}},
\]
we conclude that
\[
\sup_{\vx\in\mathcal R_i}|f_i^\star(\vx)-\hat{\vw}_i^T\tilde{\vx}|
\le
D_{\mathcal K}\frac{1}{\sqrt{cN_i}}\sqrt{N_i}\,C_A\delta^2
=
\frac{D_{\mathcal K}C_A}{\sqrt{c}}\,\delta^2.
\]
Define
\[
C_O:=\frac{D_{\mathcal K}C_A}{\sqrt{c}}.
\]
Then
\begin{equation}
\sup_{\vx\in\mathcal R_i}|f_i^\star(\vx)-\hat{\vw}_i^T\tilde{\vx}|
\le
C_O\delta^2.
\label{eq:local_ols_error}
\end{equation}

\textbf{Step 5: Global approximation bound.}
Combining \eqref{eq:local_taylor_error} and \eqref{eq:local_ols_error} yields
\[
\sup_{\vx\in\mathcal R_i}|g(\vx)-\hat{\vw}_i^T\tilde{\vx}|
\le
(C_A+C_O)\delta^2
\]
for every region \(\mathcal R_i\).
Since the regions cover \(\mathcal K\), the piecewise linear function \(f\in\mathcal F\) formed by these leafwise predictors satisfies
\[
\sup_{\vx\in\mathcal K}|f(\vx)-g(\vx)|
\le
C\delta^2,
\]
where \(C:=C_A+C_O\).
This proves \eqref{eq:approx_rate_bound}.
The universal approximation property follows immediately: for any \(\epsilon>0\), choosing \(\delta<\sqrt{\epsilon/C}\) yields
\[
\sup_{\vx\in\mathcal K}|f(\vx)-g(\vx)|<\epsilon.
\]
\end{proof}

Theorem~\ref{thm:approx_rate} provides an explicit link between partition granularity and approximation accuracy.
The \(O(\delta^2)\) rate reflects each leaf predictor matching the first-order local behavior of a \(C^2\) target, while the eigenvalue lower bound ensures that the fitted leafwise linear models preserve this order of approximation.

A simple consequence is the following depth-error scaling relation.

\begin{corollary}[Depth-error scaling]
\label{cor:depth_error_scaling}
Suppose a balanced partition scheme yields
\[
\delta(D_{\max})
\approx
\diam(\mathcal K)\,2^{-D_{\max}/d},
\]
where \(D_{\max}\) denotes the maximum tree depth.
Then, combined with Theorem~\ref{thm:approx_rate},
\[
\sup_{\vx\in\mathcal K}|f(\vx)-g(\vx)|
=
O\!\left(2^{-2D_{\max}/d}\right).
\]
Equivalently, to achieve a prescribed approximation scale \(\delta\), the required depth grows as
\[
D_{\max}=O\!\bigl(d\log(1/\delta)\bigr).
\]
\end{corollary}

\begin{proof}
Substituting
\[
\delta(D_{\max}) \approx \diam(\mathcal K)\,2^{-D_{\max}/d}
\]
into the approximation bound \eqref{eq:approx_rate_bound} gives
\[
\sup_{\vx\in\mathcal K}|f(\vx)-g(\vx)|
\le
C\,\diam(\mathcal K)^2\,2^{-2D_{\max}/d},
\]
which proves the first claim.
Solving
\[
\delta \approx \diam(\mathcal K)\,2^{-D_{\max}/d}
\]
for \(D_{\max}\) yields
\[
D_{\max}
\approx
d\log_2\!\left(\frac{\diam(\mathcal K)}{\delta}\right)
=
O\!\bigl(d\log(1/\delta)\bigr),
\]
which proves the second claim.
\end{proof}

\subsection{Ensemble-Level Risk Reduction of HRT-Boost}
\label{sec:hrtboost_theory}

\textcolor{black}{We next analyze the boosting dynamics induced by the HRT base learner under the squared loss.
The result below shows that each boosting stage yields a non-increasing empirical risk, with the magnitude of the decrease determined by the realized residual fit of the current HRT.}

\textcolor{black}{For the \(m\)-th boosting stage, define the residual vector
\[
\begin{aligned}
\mathbf{r}_m
&=
\begin{bmatrix}
r_{1,m} & \cdots & r_{N,m}
\end{bmatrix}^T
\in \mathbb{R}^{N},\\
r_{i,m}
&=
y_i - F_{m-1}(\vx_i),
\qquad i=1,\dots,N.
\end{aligned}
\]
and the corresponding base-learner response vector
\[
\mathbf{t}_m =
\begin{bmatrix}
T_m(\vx_1) & \cdots & T_m(\vx_N)
\end{bmatrix}^T
\in \mathbb{R}^{N}.
\]
We also write
\[
\begin{aligned}
\mathbf{F}_m
&=
\begin{bmatrix}
F_m(\vx_1) & \cdots & F_m(\vx_N)
\end{bmatrix}^T\in \mathbb{R}^{N},\\
\mathbf{y}
&=
\begin{bmatrix}
y_1 & \cdots & y_N
\end{bmatrix}^T\in \mathbb{R}^{N}.
\end{aligned}
\]}

\textcolor{black}{Rather than postulating a stage-wise fitting coefficient a priori, we define it directly from the empirical fit achieved by the \(m\)-th HRT.}

\textcolor{black}{\begin{theorem}[Stage-wise risk reduction]
\label{thm:stagewise_descent}
Assume that the HRT class \(\mathcal H\) contains the zero predictor.
Let \(T_m\in\mathcal H\) be an empirical minimizer of \eqref{eq:hrtboost_baselearner}.
If \(\mathbf r_m\neq \mathbf 0\), define the realized residual-fitting coefficient
\begin{equation}
\gamma_m
:=
1-\frac{\|\mathbf r_m-\mathbf t_m\|_2^2}{\|\mathbf r_m\|_2^2}.
\label{eq:boost_realized_gamma}
\end{equation}
Then \(\gamma_m\in[0,1]\), and for any learning rate \(0<\eta\le 1\),
\begin{equation}
\mathcal L(F_m)
\le
(1-\eta\gamma_m)\,\mathcal L(F_{m-1}).
\label{eq:boost_stagewise_bound}
\end{equation}
Consequently, the empirical risk is non-increasing over boosting stages, and the decrease is strict whenever \(\gamma_m>0\).
\end{theorem}}

\begin{proof}\textcolor{black}{
Since the zero predictor belongs to \(\mathcal H\), empirical optimality of \(T_m\) in \eqref{eq:hrtboost_baselearner} yields
\[
\|\mathbf r_m-\mathbf t_m\|_2^2
\le
\|\mathbf r_m-\mathbf 0\|_2^2
=
\|\mathbf r_m\|_2^2,
\]
which immediately implies \(\gamma_m\in[0,1]\).
Equivalently,
\begin{equation}
\|\mathbf r_m-\mathbf t_m\|_2^2
=
(1-\gamma_m)\|\mathbf r_m\|_2^2.
\label{eq:boost_gamma_identity}
\end{equation}}

\textcolor{black}{Next, by the definition of the squared empirical risk \eqref{eq:hrtboost_loss},
\begin{equation}
2\,\mathcal L(F_{m-1})
=
\|\mathbf y-\mathbf F_{m-1}\|_2^2
=
\|\mathbf r_m\|_2^2.
\label{eq:boost_prev_loss}
\end{equation}
After the update
\[
F_m(\vx)=F_{m-1}(\vx)+\eta\,T_m(\vx),
\]
the new residual vector is
\[
\mathbf y-\mathbf F_m
=
\mathbf r_m-\eta\,\mathbf t_m.
\]
Hence,
\begin{equation}
2\,\mathcal L(F_m)
=
\|\mathbf r_m-\eta\,\mathbf t_m\|_2^2.
\label{eq:boost_new_loss}
\end{equation}}

\textcolor{black}{Using the identity
\[
\|\mathbf r_m-\eta\mathbf t_m\|_2^2
=
(1-\eta)\|\mathbf r_m\|_2^2
+\eta\|\mathbf r_m-\mathbf t_m\|_2^2
-\eta(1-\eta)\|\mathbf t_m\|_2^2,
\]
and noting that the last term is non-positive for \(0<\eta\le 1\), we obtain
\[
\|\mathbf r_m-\eta\mathbf t_m\|_2^2
\le
(1-\eta)\|\mathbf r_m\|_2^2
+\eta\|\mathbf r_m-\mathbf t_m\|_2^2.
\]
Substituting \eqref{eq:boost_gamma_identity} gives
\[
\begin{aligned}
\|\mathbf r_m-\eta\mathbf t_m\|_2^2
&\le (1-\eta)\|\mathbf r_m\|_2^2
    + \eta(1-\gamma_m)\|\mathbf r_m\|_2^2 \\
&= (1-\eta\gamma_m)\|\mathbf r_m\|_2^2.
\end{aligned}
\]
Combining this with \eqref{eq:boost_prev_loss} and \eqref{eq:boost_new_loss} proves
\[
\mathcal L(F_m)\le (1-\eta\gamma_m)\mathcal L(F_{m-1}).
\]
The final statement follows immediately: if \(\gamma_m>0\), then \(1-\eta\gamma_m<1\).}
\end{proof}

\textcolor{black}{Theorem~\ref{thm:stagewise_descent} shows that the stage-wise decrease of HRT-Boost is governed by the empirical residual-fit coefficient \(\gamma_m\).
Larger values of \(\gamma_m\) correspond to stronger alignment between the fitted HRT and the current residual structure, and therefore to larger reductions in empirical risk.
A direct consequence is the following cumulative bound.}

\textcolor{black}{\begin{corollary}[Cumulative risk reduction]
\label{cor:cumulative_descent}
Under the assumptions of Theorem~\ref{thm:stagewise_descent},
\begin{equation}
\mathcal L(F_M)
\le
\mathcal L(F_0)\prod_{m=1}^{M}(1-\eta\gamma_m).
\label{eq:boost_product_bound}
\end{equation}
In particular, if there exists a constant \(\underline{\gamma}>0\) such that
\[
\gamma_m\ge \underline{\gamma},
\qquad m=1,\dots,M,
\]
then
\begin{equation}
\mathcal L(F_M)
\le
(1-\eta\underline{\gamma})^M\,\mathcal L(F_0).
\label{eq:boost_linear_rate}
\end{equation}
\end{corollary}}

\textcolor{black}{\begin{proof}
Applying \eqref{eq:boost_stagewise_bound} recursively over \(m=1,\dots,M\) gives
\[
\mathcal L(F_M)
\le
\mathcal L(F_0)\prod_{m=1}^{M}(1-\eta\gamma_m),
\]
which proves \eqref{eq:boost_product_bound}.
If \(\gamma_m\ge \underline{\gamma}\) for all \(m\), then
\[
\prod_{m=1}^{M}(1-\eta\gamma_m)
\le
(1-\eta\underline{\gamma})^M,
\]
and \eqref{eq:boost_linear_rate} follows immediately.
\end{proof}}

\textcolor{black}{\begin{remark}[On \(\gamma_m\)]
\label{rem:gamma_role}
The quantity \(\gamma_m\) in \eqref{eq:boost_realized_gamma} is not imposed as an a priori weak-learning assumption.
Rather, it is the realized residual-fit coefficient induced by the actual empirical fit achieved by the \(m\)-th HRT on the training residuals.
A stronger assumption is only needed when one seeks a uniform lower bound \(\gamma_m\ge \underline{\gamma}>0\), which in turn yields a geometric decay rate through Corollary~\ref{cor:cumulative_descent}.
\end{remark}}

\textcolor{black}{\begin{remark}[Connection to single-tree approximation]
\label{rem:gamma_approx}
In the approximation-theoretic setting, let \(g\) denote the underlying target function and define
\[
e_m(\vx):=g(\vx)-F_{m-1}(\vx)
\]
as the residual function at stage \(m\).
Suppose there exists an HRT \(\widetilde{T}_m\in\mathcal H\) such that
\[
\sup_{\vx\in\mathcal K}|e_m(\vx)-\widetilde{T}_m(\vx)|
\le
C_m\delta_m^2.
\]
Then, on the training set,
\[
\|\mathbf r_m-\widetilde{\mathbf t}_m\|_2^2
\le
N C_m^2\delta_m^4.
\]
Since \(T_m\) is an empirical minimizer of \eqref{eq:hrtboost_baselearner}, it follows that
\[
\|\mathbf r_m-\mathbf t_m\|_2^2
\le
N C_m^2\delta_m^4,
\]
and therefore
\[
\gamma_m
\ge
1-\frac{N C_m^2\delta_m^4}{\|\mathbf r_m\|_2^2}.
\]
This relation makes explicit how improved single-tree approximation of the current residual function translates into a larger stage-wise descent factor.
Note, however, that even if the target function \(g\) is globally \(C^2\), the residual \(e_m=g-F_{m-1}\) need not be globally \(C^2\), because \(F_{m-1}\) is piecewise linear.
Accordingly, the approximation-theoretic lower bounds on \(\gamma_m\) are most naturally formulated under piecewise-smooth regularity of \(e_m\) on a partition refining the current ensemble partition.
\end{remark}}

\begin{figure*}[t]
\vspace{-0.1in}
\begin{center}
\centerline{\includegraphics[width=0.99\textwidth]{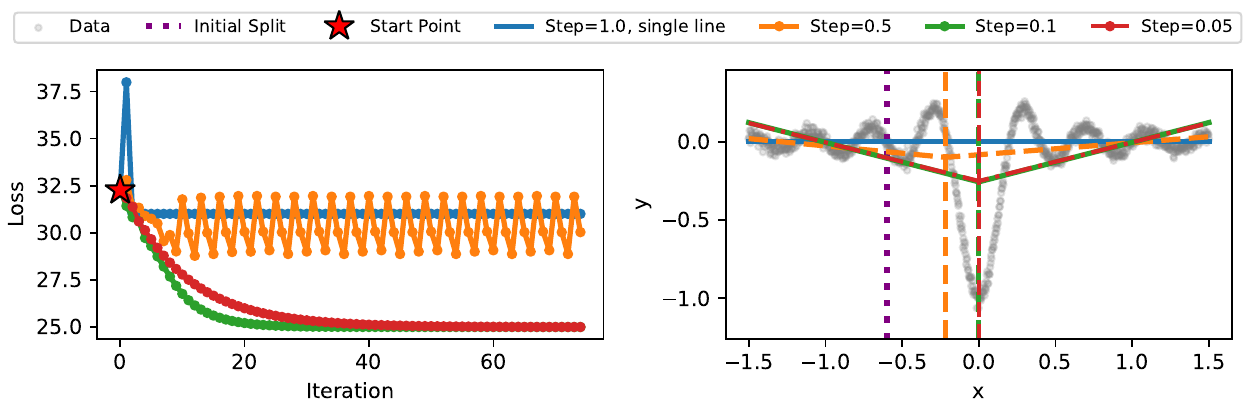}}
\caption{Node-level convergence analysis on the \texttt{sinc} function.
\textbf{Left:} Objective value per iteration for a single internal node (fixed initialization and data subset). The unit step (\(\mu=1.0\), blue) and the large step (\(\mu=0.5\), orange) do not decrease the objective monotonically in this example, and the latter gets trapped in a limit cycle. Smaller damping (\(\mu=0.1\), green; \(\mu=0.05\), red) yields more regular local Newton dynamics at this node.
\textbf{Right:} Final fitted models for this controlled experiment. With large steps, the node effectively collapses to a poor single linear fit, whereas sufficiently damped updates recover a meaningful piecewise linear approximation. Note that this figure illustrates \emph{local} node-level behaviour; full-tree performance with fallback is analysed in Appendix~E.}
\label{fig:analysis_sinc}
\end{center}
\vskip -0.2in
\end{figure*}

\section{Experiments}
\label{sec:experiments}

To validate our algorithm's effectiveness and properties, we conducted experiments addressing four key questions: (1) Does the core splitting algorithm converge efficiently as predicted? (2) Can the resulting tree model effectively approximate complex continuous functions, thereby verifying our universal approximation claim? (3) Does our method perform competitively against established baselines on real-world regression tasks? (4) Does the ensemble extension, HRT-Boost, achieve highly competitive predictive performance while maintaining structural efficiency?

Experiments were conducted using Python~3.11.7 on a workstation equipped with an Intel\textsuperscript{\textregistered} Xeon\textsuperscript{\textregistered} Gold 6530 CPU and eight NVIDIA RTX 4090 GPUs. The GPUs were used to accelerate the deep learning baselines (TabNet \cite{arik2021tabnet} and TabM \cite{gorishniy2025tabm}).

\subsection{Convergence Analysis of the Splitting Algorithm}
\label{subsec:convergence}

\textbf{Objective:} This experiment investigated the convergence dynamics of our splitting algorithm, focusing on the critical role of the step size \(\mu\). We aimed to empirically demonstrate the trade-off between convergence speed and stability by testing the algorithm on two distinct synthetic datasets: one challenging and oscillatory, and another well-behaved yet non-trivial. 
This dual approach validated our theoretical claim that while the algorithm is equivalent to a damped Newton method within fixed partitions, a damped step size (\(\mu < 1\)) is essential for robustness, while a unit step could achieve rapid convergence in stable scenarios.

\begin{figure*}[t] 
\vskip -0.1in
\begin{center}
\centerline{\includegraphics[width=0.99\textwidth]{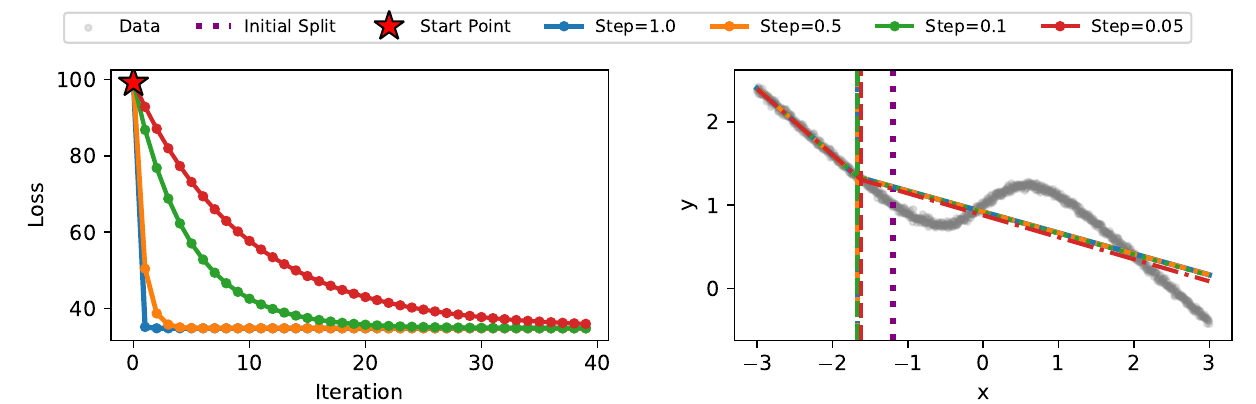}}
\caption{Node-level convergence analysis on the well-behaved \texttt{twisted\_sigmoid} function. 
\textbf{Left:} For this node, all step sizes lead to a monotone decrease of the objective. The unit Newton step (\(\mu=1.0\), blue) reaches the local minimum in the fewest iterations.
\textbf{Right:} All step sizes arrive at essentially the same high-quality piecewise linear fit around the function's inflection point. This illustrates that, on stable problems, even aggressive Newton steps can behave well at the node level.}
\label{fig:analysis_sigmoid} 
\end{center}
\vskip -0.2in
\end{figure*}

\textbf{Setup:} We generated two datasets, each of \(N=1000\) samples with Gaussian noise (\(\epsilon \sim \mathcal{N}(0,1)\)).
For the oscillatory case, data were generated from the \texttt{sinc} function, \(y = -\frac{\sin(5\pi x)}{5\pi x}+0.025\epsilon\), on \(x \in [-1.5, 1.5]\). Its multiple local extrema and sharp oscillations created a challenging optimization landscape.
For the stable case, data were generated from a smooth, nonlinear function with a distinct inflection point, \(y = \frac{2}{1+e^{-3x}} - 0.8x+0.025\epsilon\), on \(x \in [-3, 3]\). This \texttt{twisted\_sigmoid} function is an ideal well-behaved target, as it requires a nonlinear fit but lacks the sharp features that cause instability.
On both datasets, we tested four step sizes: \(\mu \in \{1.0, 0.5, 0.1, 0.05\}\).

\begin{table*}[htbp]

\scriptsize
\caption{Performance comparison on 2D and 3D synthetic functions. All values are reported as mean \( \pm \) standard deviation over 10 (2D functions) and five (3D functions) repetitions. For RMSE and MAE, lower is better (\(\downarrow\)). For R$^2$, higher is better (\(\uparrow\)). Best results are bolded. Significant improvements over the best baseline are marked with \dag{} (\(p<0.05\)).}
\label{tab:synthetic_performance_combined}
\centering
\setlength{\tabcolsep}{3pt}
\begin{tabular}{@{}lccc|lccc@{}}
\toprule
\multicolumn{4}{c|}{\texttt{sinc}} & \multicolumn{4}{c}{\texttt{twisted\_sigmoid}} \\
Model & RMSE ($\downarrow$) & MAE ($\downarrow$) & R$^2$ ($\uparrow$)&Model& RMSE ($\downarrow$) & MAE ($\downarrow$) & R$^2$ ($\uparrow$)\\
\midrule
CART & 0.0325 \( \pm \) 0.0012 & 0.0254 \( \pm \) 0.0009 & 0.9831 \( \pm \) 0.0049
& CART & 0.0312 \( \pm \) 0.0011 & 0.0249 \( \pm \) 0.0009 & 0.9976 \( \pm \) 0.0002 \\
XGBoost  & 0.0289 \( \pm \) 0.0010 & 0.0228 \( \pm \) 0.0009 & 0.9868 \( \pm \) 0.0031
& XGBoost  & 0.0286 \( \pm \) 0.0005 & 0.0226 \( \pm \) 0.0007 & 0.9980 \( \pm \) 0.0001 \\
HRT (ours) & \textbf{0.0280 \( \pm \) 0.0009} & \textbf{0.0222 \( \pm \) 0.0008} & \textbf{0.9876 \( \pm \) 0.0028}
& HRT (ours) & \textbf{0.0258 \( \pm \) 0.0004\dag{}} & \textbf{0.0205 \( \pm \) 0.0003\dag{}} & \textbf{0.9983 \( \pm \) 0.0001\dag{}} \\
\midrule
\multicolumn{4}{c|}{$f_{1}(x_1,x_2)$} & \multicolumn{4}{c}{$f_{2}(x_1,x_2)$} \\
CART & 0.8552 \( \pm \) 0.0148 & 0.6580 \( \pm \) 0.0127 & 0.9945 \( \pm \) 0.0001
& CART & 0.2721 \( \pm \) 0.0028 & 0.1972 \( \pm \) 0.0024 & 0.9301 \( \pm \) 0.0011 \\
XGBoost  & 0.3018 \( \pm \) 0.0088 & 0.2317 \( \pm \) 0.0069 & 0.9993 \( \pm \) 0.0000
& XGBoost  & 0.0859 \( \pm \) 0.0015 & 0.0681 \( \pm \) 0.0013 & 0.9930 \( \pm \) 0.0003 \\
HRT (ours) & \textbf{0.1646 \( \pm \) 0.0796\dag{}} & \textbf{0.1080 \( \pm \) 0.0362\dag{}} & \textbf{0.9998 \( \pm \) 0.0003\dag{}}
& HRT (ours) & \textbf{0.0757 \( \pm \) 0.0020\dag{}} & \textbf{0.0587 \( \pm \) 0.0016\dag{}} & \textbf{0.9946 \( \pm \) 0.0003\dag{}} \\
\midrule
\multicolumn{4}{c|}{$f_{3}(x_1,x_2)$} & \multicolumn{4}{c}{$f_{4}(x_1,x_2)$} \\
CART & 0.0752 \( \pm \) 0.0011 & 0.0587 \( \pm \) 0.0007 & 0.9832 \( \pm \) 0.0005
& CART & 0.0789 \( \pm \) 0.0014 & 0.0575 \( \pm \) 0.0005 & 0.9945 \( \pm \) 0.0002 \\
XGBoost  & 0.0537 \( \pm \) 0.0003 & 0.0428 \( \pm \) 0.0001 & 0.9914 \( \pm \) 0.0002
& XGBoost  & 0.0556 \( \pm \) 0.0005 & 0.0438 \( \pm \) 0.0003 & \textbf{0.9973 \( \pm \) 0.0001} \\
HRT (ours) & \textbf{0.0528 \( \pm \) 0.0004\dag{}} & \textbf{0.0420 \( \pm \) 0.0003\dag{}} & \textbf{0.9917 \( \pm \) 0.0002\dag{}}
& HRT (ours) & \textbf{0.0555 \( \pm \) 0.0007} & \textbf{0.0436 \( \pm \) 0.0004} & \textbf{0.9973 \( \pm \) 0.0001} \\
\bottomrule
\end{tabular}

\end{table*}

\begin{table*}[htbp]
\scriptsize
    \centering
    \caption{Average RMSE  \( \pm \) standard deviation results (lower is better) on regression tasks over five runs. The first column lists the dataset name, followed by the numbers of features, \(N_f\), and samples, \(N_s\). The first seven datasets are from \cite{zharmagambetov2021non}. The reported results for TAO (including TAO-A: axis-aligned and TAO-O: oblique) and CART were taken from \cite{zharmagambetov2021non}; DGT results were from \cite{karthikeyanlearning}; DTSemNet results were from \cite{panda2024vanilla}. For datasets with results reported in the original papers, we directly cite their results. For datasets without reported results, we reproduced the experiments under the same settings. Note that for the Ailerons, D-Elevators, and D-Ailerons datasets, the reported RMSE values are scaled by \(\times 10^{-4}\), \(\times 10^{-3}\), and \(\times 10^{-4}\) respectively. Significant improvements over the best baseline are marked with \dag{} (\(p<0.05\)). } %
    \label{tab:data_comparison} 
    \begin{tabular}
{@{}lcccccccc@{}}
        \toprule
        Dataset (\(N_f, N_s\)) & DTSemNet & DGT  & TAO-A  & TAO-O & CART & \textcolor{black}{M5} & \textcolor{black}{Linear tree} & HRT (ours)\\ 
        \midrule
Abalone (8, 4k) & 2.14 \(\pm\) 0.03 & 2.15 \(\pm\) 0.03 & 2.32 \(\pm\) 0.58 & 2.18 \(\pm\) 0.05 & 2.34 \(\pm\) 0.59 & 2.18 \(\pm\) 0.04 & 2.22 \(\pm\) 0.05 & \textbf{2.11 \(\pm\) 0.05}\\
CPUact (21, 8k) & 2.65 \(\pm\) 0.18 & 2.91 \(\pm\) 0.15 & 3.26 \(\pm\) 0.51 & 2.71 \(\pm\) 0.04 & 3.28 \(\pm\) 0.44 & 4.16 \(\pm\) 2.72 & 3.05 \(\pm\) 0.66 & \textbf{2.56 \(\pm\) 0.05} \\
Ailerons (40, 14k) & {1.66 \(\pm\) 0.01} & 1.72 \(\pm\) 0.02 & 2.55 \(\pm\) 0.00 & 1.76 \(\pm\) 0.02 & 2.85 \(\pm\) 0.57 & 1.68 \(\pm\) 0.00 & 1.75 \(\pm\) 0.00 & \textbf{1.64 \(\pm\) 0.00}\dag{} \\
CTSlice (384, 54k) & 1.45 \(\pm\) 0.12 & 2.30 \(\pm\) 0.17 & 2.66\(\pm\) 0.04 & {1.54 \(\pm\) 0.05} & 2.69 \(\pm\) 0.03 & 3.81 \(\pm\) 0.28 & 2.89 \(\pm\) 0.67 & \textbf{1.41 \(\pm\) 0.15}\\
YearPred (90, 515k) & 8.99 \(\pm\) 0.01 & 9.05 \(\pm\) 0.01 & 9.76\(\pm\)0.11 & 9.11 \(\pm\) 0.05 & 9.79 \(\pm\) 0.54 & 9.43 \(\pm\) 0.03 & 9.15 \(\pm\) 0.01 & \textbf{8.97 \(\pm\) 0.02}\dag{} \\
Concrete (8, 1k) & 7.96 \(\pm\) 0.33 & 7.43 \(\pm\) 0.24 & 7.20\(\pm\)3.17 & 7.17 \(\pm\) 0.43 & 7.22 \(\pm\) 3.13 & \textbf{6.65 \(\pm\) 0.12} & 6.93 \(\pm\) 0.74 & 6.92 \(\pm\) 0.24\\
Airfoil (5, 2k) & 3.83 \(\pm\) 0.16 & 3.72 \(\pm\) 0.10 & 2.73\(\pm\)0.62 & 3.13 \(\pm\) 0.38 & 2.75 \(\pm\) 0.62 & 3.21 \(\pm\) 0.06 & 2.81 \(\pm\) 0.12 & \textbf{2.63 \(\pm\) 0.10}\\
\midrule
Fried (10, 41k) & 1.51 \(\pm\) 0.00 & 2.27 \(\pm\) 0.09 & -- & -- & 1.96 \(\pm\) 0.01 & 1.59 \(\pm\) 0.02 & 1.10 \(\pm\) 0.00 & \textbf{1.09 \(\pm\) 0.01}\dag{} \\
D-Elevators (6, 10k) & 1.46 \(\pm\) 0.00 & 1.46 \(\pm\) 0.01 & -- & -- & 1.50 \(\pm\) 0.02 & 1.43 \(\pm\) 0.02 & 1.45 \(\pm\) 0.02 & \textbf{1.43 \(\pm\) 0.01}\\
D-Ailerons (5, 7k) & 1.76 \(\pm\) 0.00 & 1.75 \(\pm\) 0.02 & -- & -- & 1.83 \(\pm\) 0.04 & \textbf{1.65 \(\pm\) 0.03} & 1.72 \(\pm\) 0.02 & 1.68 \(\pm\) 0.02\\
Kinematics (8, 8k) & 0.168 \(\pm\) 0.000 & 0.135 \(\pm\) 0.003 & -- & -- & 0.200 \(\pm\) 0.003 & 0.175 \(\pm\) 0.004 & 0.141 \(\pm\) 0.003 & \textbf{0.102\(\pm\) 0.003}\dag{} \\
C\&C (127, 2k) & 0.201 \(\pm\) 0.000 & 0.227 \(\pm\) 0.000 & -- & -- & 0.160 \(\pm\) 0.005 & 0.145 \(\pm\) 0.004 & 0.195 \(\pm\) 0.009 & \textbf{0.140 \(\pm\) 0.004} \dag{} \\
        \bottomrule
    \end{tabular}
\end{table*}

\textbf{Results and Analysis:}
The results in Fig. \ref{fig:analysis_sinc} (oscillatory case) and Fig. \ref{fig:analysis_sigmoid} (stable case) show the role of the step size.
The results on the \texttt{sinc} function highlighted the necessity of damping for stability (Fig. \ref{fig:analysis_sinc}). The unit Newton step (\(\mu = 1.0\)) was oscillatory. Its aggressive updates caused a partition collapse within the first few iterations, making the algorithm revert to a poor global linear fit. The large damped step (\(\mu = 0.5\)), while avoiding outright collapse, became trapped in a limit cycle; the model's parameters and the corresponding data partition oscillated between a small set of states without ever converging to a fixed point. Small steps (\(\mu = 0.1, 0.05\)) proved robust, converging to a high-quality piecewise linear model that accurately captures the function's complex structure. This clearly demonstrated that for challenging problems, effective damping is not merely a heuristic but a prerequisite for success.

Results on the \texttt{twisted\_sigmoid} function (see Fig. \ref{fig:analysis_sigmoid} for visualization and detailed analysis) illustrated the efficiency benefits of the Newton updates. Here, all step sizes converged to the same solution around the function's inflection point. The unit Newton step (\(\mu = 1.0\)) exhibited the fastest convergence, reaching the optimal solution in just a few iterations. As the step size decreased, the number of iterations required for convergence predictably increased, showcasing the classic speed-stability trade-off.

These experiments characterized our algorithm's behavior as a damped Newton method.
On complex, oscillatory data, a small step size (\(\mu < 1\)) was essential to ensure robust convergence by preventing partition collapse. On well-behaved problems, a unit step size (\(\mu = 1\)) leveraged the full power of the Newton update, achieving extremely rapid convergence. This confirmed the algorithm's theoretical foundation and provided a clear practical guideline: the step size \(\mu\) serves as a crucial hyperparameter to balance convergence speed against stability, depending on the nature of the problem. 
Additional ablation studies on the fallback mechanism are reported in Appendix~E.

\subsection{Function Approximation on Synthetic Data}
\label{subsec:approximation}

\textbf{Objective:}
This experiment validated the piecewise linear approximation of our proposed oblique regression tree and connected it with Theorem~\ref{thm:approx_rate}, demonstrating its superiority in fitting complex functions.

\textbf{Setup:}
We evaluated our method on both 2D and 3D regression tasks.
For 2D tasks, we used two classic test functions: the \texttt{sinc} function \(y = -\frac{\sin(5\pi x)}{5\pi x}\)  on \(x \in [-1.5, 1.5]\) and the \texttt{twisted\_sigmoid} function \(y = \frac{2}{1+e^{-3x}} - 0.8x\)  on \(x \in [-3, 3]\).
For 3D tasks, we evaluated our method's performance on four oscillatory surface functions with complex characteristics (detailed equations available in Appendix~F), with inputs \(x_1, x_2\) in the range \([-3, 3]\).
All synthetic data were generated by adding independent and identically distributed zero-mean Gaussian noise to the true function values. Specifically, the noise standard deviation was \(0.025\) for 2D tasks and \(0.05\) for 3D tasks.
A detailed rationale for the selection of these synthetic functions and noise levels is provided in Appendix~F.
For each task, we generated a sufficient number of data points (1,000 for 2D tasks, 10,000 for 3D tasks) and split them into 70\% for training and 30\% for testing.

We compared our piecewise linear regression tree with two representative baseline methods, CART and XGBoost.
These baseline models were implemented using the \texttt{scikit-learn} library.
Hyperparameters for all models were optimized via five-fold cross-validation and grid search on the training set, with final performance reported on the testing set.
For detailed hyperparameter configurations, please refer to Appendix~G.

\textbf{2D and 3D Experimental Results:}
 The results, summarized in Table~\ref{tab:synthetic_performance_combined}, demonstrated our method's ability to effectively approximate complex functions across both 2D and 3D tasks. For a detailed visual comparison of 2D approximation performance and visualizations for $f_1$, $f_2$, $f_3$ and $f_4$, please refer to Appendix~F. 
The superior performance of our method across several synthetic functions arose from two fundamental design principles.
First, the combination of piecewise linear modeling and Newton iterative optimization allowed for precise and efficient local approximation. In contrast to traditional piecewise constant tree models, we explicitly fitted linear or planar functions within each partition and employed Newton updates to quickly converge to high-quality solutions.
Second, flexible oblique splits significantly enhanced the adaptability of feature-space partitioning. By overcoming the limitations of axis-parallel splits, they yielded decision boundaries that better aligned with the intrinsic geometry of the data, enabling more efficient and accurate region segmentation.

\subsection{Performance on Real-World Regression Datasets}
\label{subsec:real_world}

\textbf{Objective:} This experiment aimed to evaluate the practical performance of our proposed model on standard benchmark regression datasets and compare it against other well-established methods.

\textbf{Setup:} 
To assess the practical performance of the proposed method, we evaluated it on a diverse suite of publicly available regression datasets, including all seven benchmarks from \cite{zharmagambetov2021non} and several large-scale industrial datasets such as \emph{YearPred}. These datasets cover a wide spectrum of feature dimensionalities (\(N_f\)) and sample sizes (\(N_s\)), from low-dimensional, small-sample tasks to high-dimensional, large-scale scenarios (see Table~\ref{tab:data_comparison}).

We compared against the following strong baselines: DTSemNet \cite{panda2024vanilla}, DGT \cite{karthikeyanlearning}, TAO (oblique and axis-aligned) \cite{carreira2018alternating, zharmagambetov2021non}, CART \cite{breiman1984cart}, \textcolor{black}{M5 model trees \cite{wang1997inducing} and the linear trees implemented in the \texttt{linear-tree} library \cite{cerliani2022lineartree}}.

For the seven datasets from \cite{zharmagambetov2021non}, we directly cited performance figures from that work. For the remaining datasets without published results, we reproduced the results. The only exception is TAO, which is not publicly available; its reported numbers were directly taken from the original paper.
For reproduced experiments, hyperparameters were tuned via five-fold cross-validation on the training set, with final performance reported on the testing set. Specifically, for datasets with predefined train/test splits, we adhered to those original partitions. For datasets without such predefined splits, we first randomly divided the entire dataset in a 0.5:0.5 train/test ratio. For detailed hyperparameter configurations, please refer to Appendix~G.

\begin{table*}[htbp]
\scriptsize
    \centering
    \caption{
Average depths \((\Delta)\) and average number of leaves (L) over five repetitions for regression datasets. 
A dash (--) indicates that the corresponding result was not reported in the original source.
Datasets and results for TAO (including TAO-A: axis-aligned and TAO-O: oblique) and CART are from \cite{zharmagambetov2021non}, DGT results are from \cite{karthikeyanlearning}, and DTSemNet results are from \cite{panda2024vanilla}.
}
    \label{tab:tree_complexity}
    \begin{tabular}{@{}l|cc|cc|cc|cc|cc|cc@{}}
        \toprule
        Dataset (\(N_f, N_s\)) & \multicolumn{2}{c}{DTSemNet} & \multicolumn{2}{c}{DGT} & \multicolumn{2}{c}{TAO-A} & \multicolumn{2}{c}{TAO-O} & \multicolumn{2}{c}{CART} & \multicolumn{2}{c}{HRT (ours)} \\
        \cmidrule(lr){2-3} \cmidrule(lr){4-5} \cmidrule(lr){6-7} \cmidrule(lr){8-9} \cmidrule(lr){10-11} \cmidrule(lr){12-13}
        & \(\Delta\) & L & \(\Delta\) & L & \(\Delta\) & L & \(\Delta\) & L & \(\Delta\) & L & \(\Delta\) & L \\
        \midrule
        Abalone (8, 4k)          & 5.0 & -- & 6.0  & -- & 5.0  & 12.8  & 6.0 & 58.6  & 5.0  & 12.8  & \textbf{2.0} & \textbf{4.0} \\
        CPUact (21, 8k)           & 5.0 & -- & 6.0 & -- & 9.0  & 57.2  & 6.0 & 52.7  & 9.0  & 57.2  & \textbf{2.0} & \textbf{4.0} \\
        Ailerons (40, 14k)         & 5.0 & -- & 6.0 & -- & 7.0  & 15.0  & 6.0 & 60.2  & 7.0  & 15.0  & \textbf{2.0} & \textbf{4.0} \\
        CTSlice (384, 54k)          & 5.0 & -- & 10.0 & -- & 36.0 & 700.0 & \textbf{7.0} & \textbf{74.8} & 36.0 & 700.0 & \textbf{7.0} & 93.2 \\
        YearPred (90, 515k)         & 6.0 & -- & 8.0 & -- & 12.0 & 135.0 & 8.0 & 157.9 & 12.0 & 135.0 & \textbf{4.0} & \textbf{16.0} \\
        Concrete (8, 1k)         & 5.0  & -- & 6.0  & -- & 11.2 & 113.0 & 9.0 & 192.0 & 11.2 & 113.0 & \textbf{3.0} & \textbf{5.8} \\
        Airfoil (5, 2k)          & \textbf{5.0}  & -- & 6.0  & -- & 15.0 & 479.8 & 8.0 & 147.1 & 15.0 & 479.8 & \textbf{5.0} & \textbf{25.0} \\
        \bottomrule
    \end{tabular}
\end{table*}

\begin{table*}[htbp]
\scriptsize
    \centering
    \caption{\textcolor{black}{Training time (in seconds, lower is better) on three regression datasets. The numbers in parentheses represent the average number of iterations for HRT (ours).}}
    \label{tab:training_time_comparison}
    \begin{tabular}
    {@{}lcccccc@{}}
        \toprule
        Dataset (\(N_f, N_s\)) & DTSemNet & DGT & CART  & M5 & Linear tree & HRT (ours)\\ 
        \midrule
        D-Ailerons (5, 7k)& 9.557 & 60.850 & 0.008  & 0.344 & 3.601 & 0.183 (6.7) \\
        Fried (10, 41k)   & 69.979 & 395.243  & 0.171  & 8.322 & 25.743 & 2.807 (45.7) \\
        Kinematics (8, 8k)   & 16.950 & 93.690 & 0.028  & 1.447 & 7.960 & 0.790 (15.8)\\
        \bottomrule
    \end{tabular}
\end{table*}

\textbf{Results and Analysis:} 
As shown in Table~\ref{tab:data_comparison}, our method achieved the best or a highly competitive RMSE on a broad set of datasets. Among the single-tree models considered here, HRT was often among the strongest performers and yielded clear RMSE improvements on several datasets. On large-scale datasets such as YearPred, its performance remained competitive with the best reported results, indicating that the approach scales favorably beyond small and medium-sized benchmarks.
For hyperparameter configurations, please refer to Appendix~G.

In addition to RMSE, Table~\ref{tab:tree_complexity} assesses the structural complexity of the learned models in terms of \emph{average depth} (\(\Delta\)) and \emph{average number of leaves} (L) over five runs. Our method typically produced substantially shallower trees with fewer leaves compared to other single-tree baselines. For instance, on \emph{Concrete}, our model achieved competitive RMSE with a depth of only 3 and 5.8 leaves, while CART required a depth of 11.2 and 113 leaves yet attained a higher error on this dataset. This indicates that the proposed method attains a favorable trade-off between \emph{predictive performance} and \emph{transparency}. Our method also showed efficient training performance on multiple datasets (Table~\ref{tab:training_time_comparison}). The observed training efficiency is consistent with the proposed optimization scheme: each node update reduces to solving two least-squares problems under fixed partitions, while the learned trees are typically shallow and contain relatively few leaves.

\subsection{Evaluation of the HRT-Boost}
\label{subsec:hrt_boost_evaluation}

\textcolor{black}{\textbf{Objective:} Having established the efficacy of a single HRT, we evaluated its ensemble extension, HRT-Boost. The objective was to assess its predictive performance, model complexity, and computational efficiency against strong ensemble baselines.}

\textcolor{black}{\textbf{Setup:} 
We compared HRT-Boost with five widely used and highly optimized tree-based ensemble algorithms: Random Forest (RF) \cite{breiman2001random}, Scikit-Learn's Gradient Boosting Machine (Scikit-GBM) \cite{friedman2001greedy, pedregosa2011scikit}, AdaBoost \cite{freund1997decision}, XGBoost \cite{chen2016xgboost}, and LightGBM \cite{ke2017lightgbm}.
For deep tabular baselines, we included TabM \cite{gorishniy2025tabm} and TabNet \cite{arik2021tabnet}. TabM is a parameter-efficient ensemble-based MLP architecture for tabular deep learning, whereas TabNet uses sequential attention for interpretable feature selection in tabular learning.
The evaluation spanned the same suite of regression datasets utilized in Section~\ref{subsec:real_world}. To ensure a fair comparison, all models underwent rigorous hyperparameter tuning via five-fold cross-validation on the training sets. For our proposed HRT-Boost, we specifically enabled the automatic step size mechanism (\texttt{step\_size='auto'}) across all boosting stages to dynamically adapt to the changing residual landscapes. The detailed optimal hyperparameter configurations for all models and datasets are provided in Appendix~G.}

\textcolor{black}{\textbf{Results and Analysis:}
The empirical results were broadly consistent with the theoretical motivation of HRT-Boost discussed in Section~\ref{sec:hrt_boost}. We analyzed the performance across three key dimensions:}

\begin{table*}[htbp]
\scriptsize
    \centering
    \caption{\textcolor{black}{Average RMSE \( \pm \) standard deviation results (lower is better) on regression tasks over five runs. Note that for the Ailerons, D-Elevators, and D-Ailerons datasets, the reported RMSE values are scaled by \(\times 10^{-4}\), \(\times 10^{-3}\), and \(\times 10^{-4}\) respectively. Significant improvements over the best baseline are marked with \dag{} (\(p<0.05\)).} }
    \label{tab:rmse_comparison} 
    \begin{tabular}{@{}lcccccccc@{}}
        \toprule
        Dataset (\(N_f, N_s\)) & RF & Scikit-GBM & AdaBoost & XGBoost & LightGBM & TabM & TabNet & HRT-Boost (ours) \\ 
        \midrule
Abalone (8, 4k) & 2.19 \( \pm \) 0.00 & 2.20 \( \pm \) 0.01 & 2.41 \( \pm \) 0.01 & 2.21 \( \pm \) 0.00 & 2.20 \( \pm \) 0.00 & 2.12 \( \pm \) 0.00 & 2.35 \( \pm \) 0.14 & \textbf{2.08 \( \pm \) 0.00} \dag{} \\
CPUact (21, 8k) & 2.89 \( \pm \) 0.01 & 2.31 \( \pm \) 0.01 & 3.94 \( \pm \) 0.03 & 2.30 \( \pm \) 0.01 & \textbf{2.29 \( \pm \) 0.00} & 2.43 \( \pm \) 0.01 & 3.83 \( \pm \) 0.25 & 2.40 \( \pm \) 0.01 \\
Ailerons (40, 14k) & 1.77 \( \pm \) 0.00 & 1.63 \( \pm \) 0.00 & 2.23 \( \pm \) 0.02 & 1.69 \( \pm \) 0.01 & 1.63 \( \pm \) 0.00 & 1.61 \( \pm \) 0.00 & 2.20 \( \pm \) 0.58 & \textbf{1.61 \( \pm \) 0.00} \\
CTSlice (384, 54k) & 8.81 \( \pm \) 0.02 & 4.35 \( \pm \) 0.00 & 13.45 \( \pm \) 0.07 & 2.57 \( \pm \) 0.06 & 1.98 \( \pm \) 0.00 & 1.45 \(\pm\) 0.05 & 2.04 \(\pm\) 0.41 &\textbf{1.24 \( \pm \) 0.00} \dag{}\\
YearPred (90, 515k) & 9.89 \( \pm \) 0.00 & 9.11 \( \pm \) 0.00 & 12.52 \( \pm \) 0.30 & 8.96 \( \pm \) 0.00 & 8.91 \( \pm \) 0.00 & \textbf{8.69 \(\pm\) 0.01} & 8.78 \(\pm\) 0.02 & 8.76 \( \pm \) 0.00   \\
Concrete (8, 1k) & 6.23 \( \pm \) 0.04 & \textbf{4.50 \( \pm \) 0.07} & 7.55 \( \pm \) 0.07 & 4.59 \( \pm \) 0.10 & 4.61 \( \pm \) 0.00 &  5.68 \(\pm\) 0.05 & 7.11 \(\pm\) 0.32 & 5.12 \( \pm \) 0.05 \\
Airfoil (5, 2k) & 2.63 \( \pm \) 0.02 & 1.94 \( \pm \) 0.02 & 3.63 \( \pm \) 0.04 & \textbf{1.80 \( \pm \) 0.01} & 2.01 \( \pm \) 0.00 & 2.99 \( \pm \) 0.05 & 3.49 \( \pm \) 0.25 & 1.87 \( \pm \) 0.12 \\
Fried (10, 41k) & 2.00 \( \pm \) 0.00 & 1.09 \( \pm \) 0.01 & 2.00 \( \pm \) 0.01 & 1.08 \( \pm \) 0.00 & 1.09 \( \pm \) 0.00 & 1.05 \( \pm \) 0.02 & 1.09 \( \pm \) 0.01 & \textbf{1.02 \( \pm \) 0.00} \dag{}  \\
D-Elevators (6, 10k) & 1.42 \( \pm \) 0.00 & 1.41 \( \pm \) 0.00 & 1.51 \( \pm \) 0.00 & 1.40 \( \pm \) 0.00 & 1.41 \( \pm \) 0.00 & 1.40 \( \pm \) 0.01 & 1.44 \( \pm \) 0.01 & \textbf{1.40 \( \pm \) 0.00} \\
D-Ailerons (5, 7k) & 1.69 \( \pm \) 0.00 & 1.68 \( \pm \) 0.00 & 1.86 \( \pm \) 0.01 & 1.76 \( \pm \) 0.01 & 1.68 \( \pm \) 0.00 & \textbf{1.66 \( \pm \) 0.00} & 1.73 \( \pm \) 0.01 & 1.68 \( \pm \) 0.00 \\
Kinematics (8, 8k) & 0.171 \( \pm \) 0.001 & 0.152 \( \pm \) 0.001 & 0.202 \( \pm \) 0.001 & 0.126 \( \pm \) 0.001 & 0.120 \( \pm \) 0.000 & 0.081 \( \pm \) 0.001 & 0.107 \( \pm \) 0.002 & \textbf{0.079 \( \pm \) 0.000} \dag{} \\
C\&C (127, 2k) & 0.134 \( \pm \) 0.001 & 0.138 \( \pm \) 0.001 & 0.142 \( \pm \) 0.000 & 0.138 \( \pm \) 0.000 & 0.139 \( \pm \) 0.000 & 0.133 \(\pm\) 0.000 & 0.147 \(\pm\) 0.005 & \textbf{0.133 \( \pm \) 0.001} \\
        \bottomrule
    \end{tabular}
\end{table*}

\textcolor{black}{\textit{1) Predictive Accuracy:} As reported in Table~\ref{tab:rmse_comparison}, HRT-Boost delivered strong and often highly competitive predictive performance across the evaluated datasets. It achieved the lowest average RMSE on several datasets (Abalone, Ailerons, CTSlice, Fried, Kinematics, and C\&C), and remained competitive on the others. On large-scale datasets such as \emph{YearPred} and high-dimensional datasets like \emph{C\&C}, HRT-Boost compared favorably with strong baselines such as XGBoost and LightGBM, while using substantially fewer leaves. 
The base HRT can capture structured local residual patterns effectively, which supports strong stage-wise error reduction in boosting.}

\textcolor{black}{\textit{2) Model Complexity among Tree-based Models:} Beyond raw accuracy, a hallmark of HRT-Boost is its structural efficiency. Table~\ref{tab:leaf_comparison} compares the total number of leaf nodes across the ensemble models, serving as a proxy for model complexity and memory footprint. 
HRT-Boost required substantially fewer leaves than RF, XGBoost, and LightGBM across almost all datasets. For example, on the \emph{Abalone} dataset, HRT-Boost attained the best RMSE in our comparison using only 200 total leaf nodes, whereas XGBoost and LightGBM required 1175 and 1550 leaves, respectively. This compactness is a direct consequence of the hierarchical ReLU-like representation capacity (Section~\ref{sec:relu_logic}). Because each oblique split effectively functions as a nonlinear feature transformation, the base trees can remain shallow while still capturing complex data distributions.}

\textcolor{black}{\textit{3) Computational Efficiency:} To provide a fair, hardware-agnostic evaluation of computational cost, Table~\ref{tab:model_flops_comparison} compares the estimated Floating Point Operations (FLOPs) for both training and single-sample inference (detailed estimation procedures are provided in Appendix~H). 
The analysis reveals a contrast between tree-based methods and deep tabular models: TabNet and TabM require orders of magnitude more FLOPs for both training and inference. 
Within the tree-based family, HRT-Boost demonstrates practical computational efficiency. Its training cost remains in the same magnitude as RF and standard gradient boosting. Its inference FLOPs per sample are lightweight, showing that the structural compactness of HRT-Boost translates directly into efficient deployment.}

\textcolor{black}{In summary, HRT-Boost successfully scales the piecewise linear oblique tree into a powerful functional gradient boosting machine. Overall, HRT-Boost combines strong predictive accuracy with a clear advantage in compactness, making it a competitive alternative to traditional ensemble methods.}

\begin{table*}[htbp]
\scriptsize
    \centering
    \caption{\textcolor{black}{Average total number of leaf nodes for tree-based methods.} }
    \label{tab:leaf_comparison} 
    \begin{tabular}{@{}lcccccc@{}}
        \toprule
        Dataset (\(N_f, N_s\)) & RF & Scikit-GBM & AdaBoost & XGBoost & LightGBM & HRT-Boost (ours) \\ 
        \midrule
        Abalone (8, 4k) & 12610 & 762 & 400 & 1175 & 1550 & \textbf{200} \\
        CPUact (21, 8k) & 9273 & 2223 & \textbf{1071} & 6560 & 4650 & 2077 \\
        Ailerons (40, 14k) & 16752 & 2222 & \textbf{399} & 735 & 4650 & 400 \\
        CTSlice (384, 54k) & 4759 & 2400 & \textbf{1197} & 8525& 9450 & 3194 \\
        YearPred (90, 515k) & 4800 & 2387 & \textbf{395} & 9466 & 9450 & 800 \\
        Concrete (8, 1k) & 16079 & 2231 & \textbf{1193} & 6312 & 4630 & 1646 \\
        Airfoil (5, 2k) & 15261 & 2313 & 1189 & 6407 & 4278 & \textbf{827} \\
        Fried (10, 41k) & 19194 & 2376 & 1200 & 8989 & 4650 & \textbf{1181} \\
        D-Ailerons (5, 7k) & 11752 & 768 & 400 & \textbf{285} & 1550 & 580 \\
        D-Elevators (6, 10k) & 16525 & 780 & \textbf{400} & 1111 & 1550 & 640 \\
        Kinematics (8, 8k) & 16917 & 2270 & \textbf{1198} & 8025 & 9450 & 2275 \\
        C\&C (127, 2k) & 10760 & 772 & 399 & 388 & 2690 & \textbf{284} \\
        \bottomrule
    \end{tabular}
\end{table*}

\begin{table*}[htbp]
\scriptsize
    \centering
    \caption{\textcolor{black}{Computational cost analysis: estimated training FLOPs and single-sample inference FLOPs. (K, M, and B denotes $10^3$, $10^6$, and $10^9$ respectively.)}}
    \label{tab:model_flops_comparison} 
    \begin{tabular}{@{}lcccccccc@{}}
        \toprule
        Dataset & RF & Scikit-GBM & AdaBoost & XGBoost & LightGBM & TabM & TabNet & HRT-Boost (ours) \\ 
        \midrule
        \multicolumn{9}{c}{\textbf{Training FLOPs}} \\
        \midrule
        Abalone (8, 4k) & 874.44M & 262.29M & 256.88M & 211.43M & 291.41M & 690.35B & 11.35B & 349.40M \\
        Airfoil (5, 2k) & 131.14M & 121.48M & 121.04M & 293.95M & 136.72M & 123.01B & 4.84B & 63.88M \\
        Kinematics (8, 8k) & 1.51B & 1.38B & 1.33B & 1.04B & 1.18B & 674.68B & 35.42B & 2.66B \\
        Ailerons (40, 14k) & 14.63B & 13.01B & 4.10B & 4.13B & 8.48B & 1248.78B & 83.22B & 18.50B \\
        D-Ailerons (5, 7k) & 843.80M & 241.59M & 236.24M & 141.17M & 184.21M & 583.76B & 8.91B & 294.37M \\
        D-Elevators (6, 10k) & 1.39B & 400.48M & 392.14M &301.17M & 287.45M & 780.80B & 30.67B & 341.33M \\
        Fried (10, 41k) & 11.77B & 10.29B & 9.92B & 3.45B & 5.17B & 3370.11B & 110.77B & 10.97B \\
        \midrule
        \multicolumn{9}{c}{\textbf{Inference FLOPs / sample}} \\
        \midrule
        Abalone (8, 4k) & 3.13K & 694 & 832 & 2.98K & 1.75K & 2.20M & 36.02K & 3.35K \\
        Airfoil (5, 2k) & 3.20K & 2.08K & 2.73K & 5.39K & 4.93K & 1.09M & 42.73K & 4.85K \\
        Kinematics (8, 8k) & 3.25K & 2.07K & 2.73K & 5.55K & 5.47K & 1.10M & 57.31K & 13.51K \\
        Ailerons (40, 14k) & 3.25K & 2.06K & 832 & 1.32K & 4.45K & 1.16M & 77.15K & 16.45K \\
        D-Ailerons (5, 7k) & 3.12K & 695 & 832 & 781 & 1.47K & 1.09M & 16.57K & 5.47K \\
        D-Elevators (6, 10k) & 3.24K & 697 & 832 & 2.93K & 1.38K & 1.09M & 42.74K & 3.37K \\
        Fried (10, 41k) & 3.30K & 2.10K & 2.73K & 5.65K & 4.22K & 1.10M & 36.02K & 13.29K \\
        \bottomrule
    \end{tabular}
\end{table*}

\section{Conclusion}

In this paper, we introduced the Hinge Regression Tree (HRT), a novel oblique regression tree framework that formulates node splitting as a nonlinear least-squares optimization. By utilizing a hinge-based envelope of linear models, the HRT achieves ReLU-like representation capacity within a compact tree structure. We demonstrated that our iterative splitting procedure is equivalent to a damped Newton method, providing a theoretical foundation for node-level optimization. We established a universal approximation property for the HRT model class, providing an explicit $O(\delta^2)$ convergence rate.

Building upon the single-tree HRT, we proposed HRT-Boost, an ensemble variant that integrates HRTs into a gradient boosting framework.
Our experimental results on both synthetic and real-world benchmarks showed that a single HRT is highly competitive with existing single-tree baselines, while HRT-Boost achieved strong predictive accuracy and often compared favorably with powerful libraries such as XGBoost and LightGBM. Importantly, HRT-Boost often achieved this level of performance with substantially lower model complexity, requiring fewer leaves and shallower structures than traditional boosting methods on many of the evaluated datasets. This combination of predictive performance and structural parsimony makes our approach an alternative for applications where both accuracy and compactness are important.

\bibliographystyle{IEEEtran}
\bibliography{hrt2026}

\begin{IEEEbiography}
[{\includegraphics[width=1in,height=1.25in,clip,keepaspectratio] {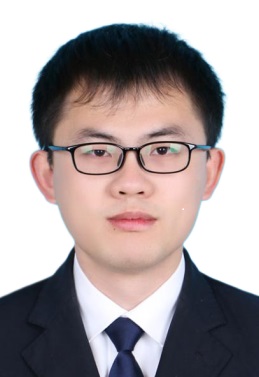}}]{Hongyi Li}
received his B.S. degree in Automation and M.S. degree in Control Science and Engineering from Harbin Institute of Technology, China, in 2020 and 2023, respectively. He is currently working toward the Ph.D. degree in Control Science and Engineering from Harbin Institute of Technology, Shenzhen, China. His research interests include optimization, decision trees, piecewise linear approximation and energy conservation in buildings. \end{IEEEbiography}

\begin{IEEEbiography}[{\includegraphics[width=1in,height=1.25in,clip,keepaspectratio]{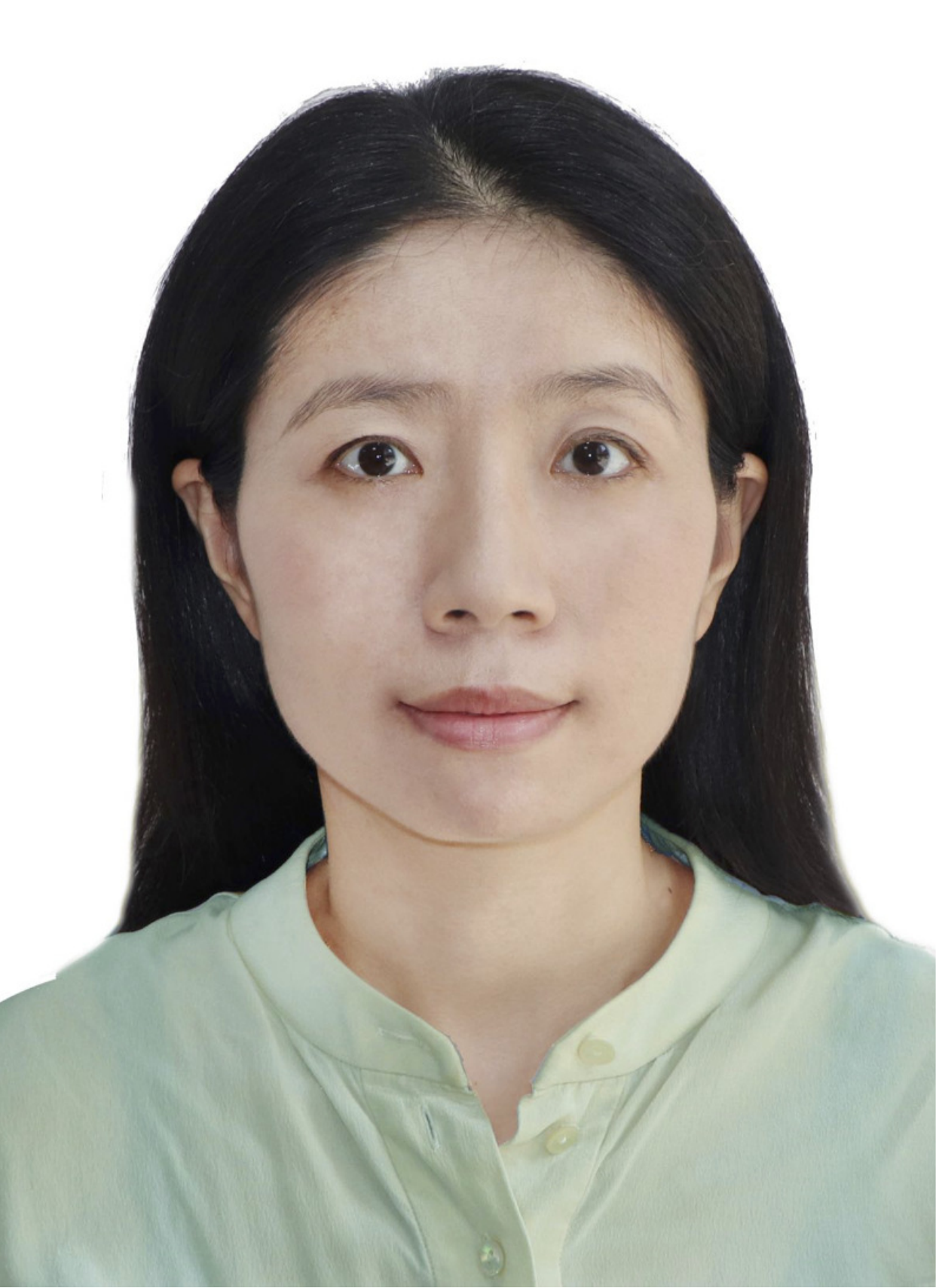}}]{Jun Xu} 
received her B.S. degree in Control Science and Engineering from Harbin Institute of Technology, Harbin, China, in 2005 and Ph.D. degree in Control Science and Engineering from Tsinghua University, China, in 2010. Currently, she is a professor in School of Intelligence Science and Engineering, Harbin Institute of Technology, Shenzhen, China. Her research interests include piecewise linear functions and their applications in machine learning, nonlinear system identification and control.\end{IEEEbiography}

\begin{IEEEbiography}[{\includegraphics[width=1in,height=1.25in,clip,keepaspectratio]{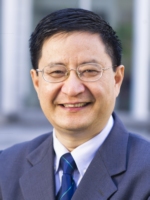}}]{Hong Yan} 
received his PhD degree from Yale University. He was Professor of Imaging Science at the University of Sydney and currently is Wong Chun Hong Professor of Data Engineering and Chair Professor of Computer Engineering at City University of Hong Kong. Professor Yan's research interests include AI, bioinformatics, pattern recognition, and signal and image processing. He has over 600 journal and conference publications in these areas. Professor Yan is an IEEE Fellow and IAPR Fellow. He received the 2016 Norbert Wiener Award from the IEEE SMC Society for contributions to image and biomolecular pattern recognition techniques. He is a Foreign Member of the European Academy of Sciences and Arts and a Fellow of the US National Academy of Inventors.\end{IEEEbiography}

\clearpage
\newpage
\onecolumn

\appendices

\startcontents[appendix]
\printcontents[appendix]{}{1}{}

\onecolumn

\section{Detailed Derivations for Equivalence of Iterative Optimization to Newton Method}
\label{sec:appendix_newton_en}

This appendix provides a detailed mathematical derivation demonstrating the equivalence of our iterative optimization procedure to a Newton method.

Recall from Section~III that at any internal node of the tree $\mathbb{D}_t$, for all $\vx_j \in \mathbb{D}_t$, we aim to minimize the nonlinear least squares objective function:
\[
V(\vtheta) = \frac{1}{2} \sum\limits_{\vx_j \in \mathbb{D}_t} \left( y_j - h(\vx_j, \vtheta) \right)^2,
\]
where \(\vtheta = [\vtheta_{t_1}^T, \vtheta_{t_2}^T]^T \in \R^{2(d+1)}\) is the total parameter vector. The function \(h(\vx_j, \vtheta)\) is defined as \(h(\vx_j, \vtheta) = \max \left( \tilde{\vx}_j^T \vtheta_{t_1}, \tilde{\vx}_j^T \vtheta_{t_2} \right)\) or \(h(\vx_j, \vtheta) = \min \left( \tilde{\vx}_j^T \vtheta_{t_1}, \tilde{\vx}_j^T \vtheta_{t_2} \right)\). For clarity, we will provide the detailed derivation for the \(\max\) form; the derivation for the \(\min\) form follows analogously with adjusted partition definitions. As discussed in Section~III, directly minimizing \(V(\vtheta)\) is challenging due to the non-differentiability of the hinge functions at the switching points.

Our approach can be interpreted as an iterative procedure equivalent to the Newton method within fixed partitions. We first define dynamic partitions of the data based on the current parameters \(\vtheta\), as introduced in Section~III:
\begin{align*}
    \calS_1(\vtheta) &= \{\vx_j \in \mathbb{D}_t \mid \tilde{\vx}_j^T \vtheta_{t_1} \ge \tilde{\vx}_j^T \vtheta_{t_2} \}, \\
    \calS_2(\vtheta) &= \mathbb{D}_t \setminus \calS_1(\vtheta).
\end{align*}
For a fixed partition \((\calS_1, \calS_2)\), the objective function is differentiable with respect to \(\vtheta_{t_1}\) and \(\vtheta_{t_2}\) within their respective domains. Under this condition, we can compute the gradient \(\nabla V\) and the Hessian matrix \(\nabla^2 V\).

\paragraph{Gradient Calculation.}
For almost all \(\vtheta\) (where no data point lies exactly on the boundary \(\tilde{\vx}_j^T \vtheta_{t_1} = \tilde{\vx}_j^T \vtheta_{t_2}\)), the model function \(h(\vx_j, \vtheta)\) is differentiable at each data point. Specifically, for \(h(\vx_j, \vtheta) = \max \left( \tilde{\vx}_j^T \vtheta_{t_1}, \tilde{\vx}_j^T \vtheta_{t_2} \right)\),
\[
\nabla h(\vx_j, \vtheta) = \begin{cases}
    \begin{pmatrix} \tilde{\vx}_j \\ \mathbf{0} \end{pmatrix} & \text{if } \tilde{\vx}_j^T \vtheta_{t_1} > \tilde{\vx}_j^T \vtheta_{t_2}, \\
    \begin{pmatrix} \mathbf{0} \\ \tilde{\vx}_j \end{pmatrix} & \text{if } \tilde{\vx}_j^T \vtheta_{t_1} < \tilde{\vx}_j^T \vtheta_{t_2}.
\end{cases}
\]
Thus, the gradient of the objective function is:
\[
\nabla V(\vtheta) = \sum_{j=1}^{N} (y_j - h(\vx_j, \vtheta)) (-\nabla h(\vx_j, \vtheta)) = - \sum_{j=1}^{N} (y_j - h(\vx_j, \vtheta)) \nabla h(\vx_j, \vtheta).
\]
Substituting \(\nabla h(\vx_j, \vtheta)\) and summing over the respective partitions:
\[
\nabla V(\vtheta) = - \begin{pmatrix}
\sum_{(\vx_j, y_j) \in \calS_1(\vtheta)} \tilde{\vx}_j (y_j - \tilde{\vx}_j^T \vtheta_{t_1}) \\
\sum_{(\vx_j, y_j) \in \calS_2(\vtheta)} \tilde{\vx}_j (y_j - \tilde{\vx}_j^T \vtheta_{t_2})
\end{pmatrix}.
\]

\paragraph{Hessian Matrix Calculation.}
For a generic nonlinear least squares problem, the Hessian matrix is given by:
\[
\nabla^2 V(\vtheta) = \sum_j [(\nabla h(\vx_j, \vtheta))(\nabla h(\vx_j, \vtheta))^T - (y_j - h(\vx_j, \vtheta))\nabla^2 h(\vx_j, \vtheta)].
\]
However, in our specific case, because \(h(\vx_j, \vtheta)\) is locally linear within each fixed partition (i.e., \(h(\vx_j, \vtheta)\) is either \(\tilde{\vx}_j^T \vtheta_{t_1}\) or \(\tilde{\vx}_j^T \vtheta_{t_2}\)), its second derivative \(\nabla^2 h(\vx_j, \vtheta)\) is zero. This causes the second term of the Hessian to vanish, meaning the expression \(\sum_j (\nabla h_j)(\nabla h_j)^T\), often used as the Gauss-Newton approximation, becomes the exact Hessian.

Therefore, the exact Hessian matrix is:
\[
\nabla^2 V(\vtheta) = \begin{pmatrix}
\sum_{(\vx_j, y_j) \in \calS_1(\vtheta)} \tilde{\vx}_j \tilde{\vx}_j^T & \mathbf{0} \\
\mathbf{0} & \sum_{(\vx_j, y_j) \in \calS_2(\vtheta)} \tilde{\vx}_j \tilde{\vx}_j^T
\end{pmatrix},
\]
where \(\mathbf{0}\) is a \((d+1) \times (d+1)\) zero matrix.

\paragraph{Newton Update and OLS Equivalence.}
The Newton update at iteration \(k\) is generally given by \(\vtheta^{(k+1)} = \vtheta^{(k)} - \mu [\nabla^2 V(\vtheta^{(k)})]^{-1} \nabla V(\vtheta^{(k)})\).
Let \(\vtheta^{(k)} = [\vtheta_{t_1}^{(k)T}, \vtheta_{t_2}^{(k)T}]^T\). We examine the updates for \(\vtheta_{t_1}\) and \(\vtheta_{t_2}\) separately.

For \(\vtheta_{t_1}\), the update is:
\begin{align*}
\vtheta_{t_1}^{(k+1)} &= \vtheta_{t_1}^{(k)} - \mu \left( \sum_{(\vx_j, y_j) \in \calS_1(\vtheta^{(k)})} \tilde{\vx}_j \tilde{\vx}_j^T \right)^{-1} \left( - \sum_{(\vx_j, y_j) \in \calS_1(\vtheta^{(k)})} \tilde{\vx}_j (y_j - \tilde{\vx}_j^T \vtheta_{t_1}^{(k)}) \right) \\
&= \vtheta_{t_1}^{(k)} + \mu \left( \sum_{(\vx_j, y_j) \in \calS_1^{(k)}} \tilde{\vx}_j \tilde{\vx}_j^T \right)^{-1} \left( \sum_{(\vx_j, y_j) \in \calS_1^{(k)}} \tilde{\vx}_j y_j - \sum_{(\vx_j, y_j) \in \calS_1^{(k)}} \tilde{\vx}_j \tilde{\vx}_j^T \vtheta_{t_1}^{(k)} \right) \\
&= \vtheta_{t_1}^{(k)} + \mu \left[ \left( \sum_{(\vx_j, y_j) \in \calS_1^{(k)}} \tilde{\vx}_j \tilde{\vx}_j^T \right)^{-1} \left( \sum_{(\vx_j, y_j) \in \calS_1^{(k)}} \tilde{\vx}_j y_j \right) - \vtheta_{t_1}^{(k)} \right].
\end{align*}
Let \(\vtheta_{\text{OLS},1}^{(k)} = \left( \sum_{(\vx_j, y_j) \in \calS_1^{(k)}} \tilde{\vx}_j \tilde{\vx}_j^T \right)^{-1} \left( \sum_{(\vx_j, y_j) \in \calS_1^{(k)}} \tilde{\vx}_j y_j \right)\) be the OLS solution for the data subset \(\calS_1^{(k)}\). Then the Newton update can be written as:
\[
\vtheta_{t_1}^{(k+1)} = \vtheta_{t_1}^{(k)} + \mu (\vtheta_{\text{OLS},1}^{(k)} - \vtheta_{t_1}^{(k)}).
\]
A similar derivation holds for \(\vtheta_2\):
\[
\vtheta_{t_2}^{(k+1)} = \vtheta_{t_2}^{(k)} + \mu (\vtheta_{\text{OLS},2}^{(k)} - \vtheta_{t_2}^{(k)}),
\]
where \(\vtheta_{\text{OLS},2}^{(k)}\) is the OLS solution for the data subset \(\calS_2^{(k)}\).

In our algorithm, the step size \(\mu\) is a tunable parameter. When a unit step size is employed (i.e., \(\mu=1\)), the update formulae simplify to:
\[
\begin{aligned}
    \vtheta_{t_1}^{(k+1)} &= \vtheta_{\text{OLS},1}^{(k)}, \\
    \vtheta_{t_2}^{(k+1)} &= \vtheta_{\text{OLS},2}^{(k)}.
\end{aligned}
\]
This demonstrates that when a unit step size (\(\mu=1\)) is employed, the alternating procedure of model fitting (performing OLS on the current partitions) and partitioning (reassigning data points based on the new model parameters), as presented in Algorithm \ref{alg:split_en}, is precisely equivalent to executing a unit-step Newton method to minimize the original nonlinear objective function \(V(\vtheta)\). This choice of \(\mu=1\) is common for this type of problem due to its direct connection to the optimal OLS solutions within fixed partitions.

\section{Details on Regularization and Fallback Strategies}
\label{sec:appendix_details}

\paragraph{Incorporating Ridge Regression (L2 Regularization).}
To enhance model robustness and effectively handle situations like multicollinearity or high-dimensional data, we introduce optional ridge regression (L2 regularization) in all OLS fitting steps within this algorithm. Ridge regression modifies the OLS objective by adding a penalty on the L2-norm of the model coefficients, i.e., minimizing:
\[ \min_{\vtheta} \frac{1}{2} \sum_{j=1}^{N_k} \left( y_j - \tilde{\vx}_j^T \vtheta \right)^2 + \frac{\alpha}{2} \sum_{i=0}^{d-1} \theta_i^2 \]
where \(\alpha \ge 0\) is the regularization strength parameter. Notably, the bias term \(\theta_d\) (corresponding to the constant 1 in the augmented feature vector) is typically not regularized. Given a design matrix \(X\) (already augmented with a column of ones for the bias) and a target vector \(\mathbf{y}\), the ridge regression solution with regularization parameter \(\alpha \ge 0\) is given by:
\[ \vtheta = (X^T X + \alpha I_0)^{-1} X^T \mathbf{y} \]
where \(I_0\) is an identity matrix with its last diagonal element (corresponding to the bias term) set to zero, ensuring the bias is not regularized. This regularization mechanism effectively stabilizes parameter estimation, especially when the training data matrix \(X^T X\) might be near-singular, which is particularly important in high-dimensional feature spaces or with limited sample sizes. The regularization strength \(\alpha\) is tuned as a hyperparameter during model training.

\paragraph{Fallback Strategy for Non-convergent Nodes.}
If the iterative optimization does not converge within $T_{\max}$ iterations, we fall back to a simple median split. A feature dimension $k$ is chosen at random, and the data are split at the median $m_k$ of this dimension.
This guarantees progress in tree growth and often decomposes a hard global problem
into smaller subproblems that converge quickly under our Newton updates.

\section{Parameter Initialization Strategy and Detailed Algorithms}
\label{sec:appendix_algorithms}
This section provides the detailed pseudocode for the optimal node-splitting and tree-building procedures.


\subsection{Parameter Initialization Strategy}
\label{appendix:initialization}

The first step of Algorithms~\ref{alg:split_min} and~\ref{alg:split_en} requires the initialization of the parameter vectors $\theta_1$ and $\theta_2$. A reasonable and stable initialization is important for the convergence of the subsequent alternating optimization. We employ a simple data-driven heuristic combined with a small fallback procedure.

The initialization procedure is as follows:

\begin{enumerate}
    \item \textbf{Heuristic initial partition.}  
    At the current node, we identify the feature dimension with the \emph{largest range} in the local data. We then use the \emph{median} value of this feature as a pivot to split the samples into two subsets. This aims to create an initial partition along the direction of greatest variation, which typically yields a more balanced and meaningful starting point than a fully random split.
    \item \textbf{Partition-based initial (ridge) regression.}  
    Given these two subsets, we independently solve a (ridge-regularized) least-squares problem on each subset to obtain the initial parameter vectors $\theta_1$ and $\theta_2$. In our implementation we require that each subset contains at least two samples. Otherwise we treat the split as too small for a reliable local regression.
    \item \textbf{Fallback for degenerate or very small partitions.}  
    In edge cases, such as when the node contains very few samples, the features are nearly constant, or one of the subsets after the median split is too small, we fall back to a simple global initialization. Specifically, we first compute a \emph{global} ridge regression solution $\theta_{\text{global}}$ using all samples at the node (or a constant predictor based on the mean response if the design matrix is degenerate). We then construct two initial parameter vectors $\theta_1$ and $\theta_2$ by adding small independent random perturbations to $\theta_{\text{global}}$, which ensures a non-trivial starting point even under poor initial splits.
    \item \textbf{Ensuring parameter diversity.}  
    Finally, we check whether the resulting $\theta_1$ and $\theta_2$ are nearly identical. If so, we add another small perturbation (and, in a rare corner case, a tiny offset to the first coefficient) to break the symmetry, so that the subsequent alternating optimization can proceed effectively.
\end{enumerate}

This initialization strategy provides a data-dependent yet lightweight starting point for the node-wise optimization, and makes the overall splitting procedure more robust to degenerate partitions and small-sample regimes.

\subsection{Detailed Algorithms}

\begin{algorithm}[H]
\caption{Find Optimal Split (Min Variant)}
\label{alg:split_min}
\begin{algorithmic}[1]
\Require Dataset $\mathcal{S}=\{(\vx_j,y_j)\}_{j=1}^N$, Max iterations $T_{\max}$, Step size $\mu$, Tolerance $\epsilon$
\Ensure Optimal parameters $\vtheta_1^*, \vtheta_2^*$
\State Initialize \(\vtheta_1^{(0)}, \vtheta_2^{(0)}\) and partitions \(\mathcal{S}_1, \mathcal{S}_2\).
\For{$t=1,\dots,T_{\max}$}
    \State \textit{// Model Fitting Step}
    \State Compute $\vtheta_{\text{OLS},1}^{(k)}$ and $\vtheta_{\text{OLS},2}^{(k)}$ \textbf{(with optional ridge regularization)}
    \State $\vtheta_1^{(k+1)} \gets \vtheta_1^{(k)} + \mu \big(\vtheta_{\text{OLS},1}^{(k)} - \vtheta_1^{(k)}\big) $ \Comment{Newton update with step size \(\mu\)}
    \State $\vtheta_2^{(k+1)} \gets \vtheta_2^{(k)} + \mu \big(\vtheta_{\text{OLS},2}^{(k)} - \vtheta_2^{(k)}\big)$ \Comment{Newton update with step size \(\mu\)}
    \State \textit{// Partitioning Step (min: assign to the smaller branch)}
    \State $\mathcal{S}_{1,\text{new}}\gets\emptyset,\quad \mathcal{S}_{2,\text{new}}\gets\emptyset$
    \ForAll{$(\vx_j,y_j)\in\mathcal{S}$}
        \If{$\tilde{\vx}_j^T \vtheta_1 \le \tilde{\vx}_j^T \vtheta_2$}
            \State $\mathcal{S}_{1,\text{new}} \gets \mathcal{S}_{1,\text{new}} \cup \{(\vx_j,y_j)\}$
        \Else
            \State $\mathcal{S}_{2,\text{new}} \gets \mathcal{S}_{2,\text{new}} \cup \{(\vx_j,y_j)\}$
        \EndIf
    \EndFor
    
    \If{$\|\vtheta_1^{(k+1)}-\vtheta_1^{(k)}\|+\|\vtheta_2^{(k+1)}-\vtheta_2^{(k)}\|<\epsilon$}\Comment{Check parameter changes}
        \State \textbf{break} \Comment{Convergence}
    \EndIf
    \State $\mathcal{S}_1 \gets \mathcal{S}_{1,\text{new}},\quad \mathcal{S}_2 \gets \mathcal{S}_{2,\text{new}}$
\EndFor
\State \textbf{return} $\vtheta_1,\vtheta_2$
\end{algorithmic}
\end{algorithm}

\begin{algorithm}[H]
\caption{Find Optimal Split (Max Variant)}
\label{alg:split_en}
\begin{algorithmic}[1]
\Require Dataset \(\calS = \{(\vx_j, y_j)\}_{j=1}^N\), Max iterations \(T_{max}\), Step size \(\mu\)
\Ensure Optimal parameters \(\vtheta_1^*, \vtheta_2^*\)
\State Initialize \(\vtheta_1^{(0)}, \vtheta_2^{(0)}\) and partitions \(\mathcal{S}_1, \mathcal{S}_2\).
\For{$t = 1, \dots, T_{max}$}
    \State // \textit{Model Fitting Step }
    \State Compute \(\vtheta_{\text{OLS},1}^{(k)}\) and  \(\vtheta_{\text{OLS},2}^{(k)}\) (\textbf{with optional ridge regularization})
    \State $ \vtheta_1^{(k+1)} \gets \vtheta_1^{(k)} + \mu (\vtheta_{\text{OLS},1}^{(k)} - \vtheta_1^{(k)})$ \Comment{Newton update with step size \(\mu\)}
    \State $ \vtheta_2^{(k+1)} \gets \vtheta_2^{(k)} + \mu (\vtheta_{\text{OLS},2}^{(k)} - \vtheta_2^{(k)})$ \Comment{Newton update with step size \(\mu\)}
    \State // \textit{Partitioning Step (max: assign to the larger branch)}
    \State $\calS_{1, \text{new}} \gets \emptyset, \quad \calS_{2, \text{new}} \gets \emptyset$
    \ForAll{$(\vx_j, y_j) \in \calS$}
        \If{$\tilde{\vx}_j^T\vtheta_1 \ge \tilde{\vx}_j^T\vtheta_2$}
            \State $\calS_{1, \text{new}} \gets \calS_{1, \text{new}} \cup \{(\vx_j, y_j)\}$
        \Else
            \State $\calS_{2, \text{new}} \gets \calS_{2, \text{new}} \cup \{(\vx_j, y_j)\}$
        \EndIf
    \EndFor
    \If{$||\vtheta_1^{(k+1)}-\vtheta_1^{(k)}||+||\vtheta_2^{(k+1)}-\vtheta_2^{(k)}||<\epsilon$} \Comment{Check parameter changes}
        \State \textbf{break} \Comment{Convergence}
    \EndIf
    \State $\calS_1 \gets \calS_{1, \text{new}}, \quad \calS_2 \gets \calS_{2, \text{new}}$
\EndFor

\State \textbf{return} $\vtheta_1, \vtheta_2$
\end{algorithmic}
\end{algorithm}

\begin{algorithm}[H]
\caption{Find Optimal Split (Min-or-Max by RMSE)}
\label{alg:select_minmax}
\begin{algorithmic}[1]
\Require Dataset $\mathcal{S}=\{(\vx_j,y_j)\}_{j=1}^N$, $T_{\max}$, Step size $\mu$, Tolerance $\epsilon$
\Ensure Best parameters $(\vtheta_1^*,\vtheta_2^*)$, model type $\in\{\text{min},\text{max}\}$
\State $(\vtheta^{\max}_1,\vtheta^{\max}_2) \gets$ run Algorithm \ref{alg:split_en} (max variant) on $\mathcal{S}$ with $(T_{\max},\mu,\epsilon)$
\State $(\vtheta^{\min}_1,\vtheta^{\min}_2) \gets$ run Algorithm \ref{alg:split_min} (min variant) on $\mathcal{S}$ with $(T_{\max},\mu,\epsilon)$
\State Compute $\mathrm{RMSE}_{\max} \gets \sqrt{\frac{1}{N}\sum_{j=1}^N \big(y_j - \max(\tilde{\vx}_j^T\vtheta^{\max}_1,\ \tilde{\vx}_j^T\vtheta^{\max}_2)\big)^2}$
\State Compute $\mathrm{RMSE}_{\min} \gets \sqrt{\frac{1}{N}\sum_{j=1}^N \big(y_j - \min(\tilde{\vx}_j^T\vtheta^{\min}_1,\ \tilde{\vx}_j^T\vtheta^{\min}_2)\big)^2}$
\If{$\mathrm{RMSE}_{\min} < \mathrm{RMSE}_{\max}$}
    \State \textbf{return} $\vtheta^{\min}_1,\vtheta^{\min}_2$
\Else
    \State \textbf{return} $\vtheta^{\max}_1,\vtheta^{\max}_2$
\EndIf
\end{algorithmic}
\end{algorithm}

\begin{algorithm}[H]
\caption{Build Oblique Regression Tree }
\label{alg:build_en}
\begin{algorithmic}[1]
\Require Dataset \(\calS\), current depth \(d\), hyperparameters (max depth \(D_{max}\), min samples \(N_{min}\), RMSE threshold \(\tau_{RMSE}\), step size \(\mu\))
\Ensure Root node of a (sub)tree

\State // \textit{Check stopping criteria for creating a leaf}
\State Fit a single linear model for the current node: \(\vtheta_{leaf} \gets \argmin_{\vtheta} \sum_{(\vx,y) \in \calS} (y - \tilde{\vx}^T\vtheta)^2\) \Comment{Solve via OLS (\textbf{with optional ridge regularization})}
\State Calculate RMSE: \( \text{RMSE} = \sqrt{\frac{1}{|\calS|} \sum_{(\vx,y) \in \calS} (y - \tilde{\vx}^T\vtheta_{leaf})^2} \)

\If{\(d \ge D_{max}\) or \(|\calS| < N_{min}\) or \(\text{RMSE} < \tau_{RMSE}\)} \Comment{Check depth, samples, or RMSE}
    \State Create a leaf node.
    \State Store \(\vtheta_{leaf}\) in the node and \textbf{return} the node.
\EndIf
\State

\State // \textit{Proceed with splitting the node}
\State \(\vtheta_1^*, \vtheta_2^* \gets \texttt{FindOptimalSplit}(\calS, \mu)\)
\State \(\calS_L \gets \{ (\vx,y) \in \calS \mid \tilde{\vx}^T\vtheta_1^* \ge \tilde{\vx}^T\vtheta_2^* \}\)
\State \(\calS_R \gets \calS \setminus \calS_L\)
\State

\If{\(|\calS_L| < N_{min}\) or \(|\calS_R| < N_{min}\)} \Comment{Ineffective split, create leaf}
    \State Create a leaf node.
    \State Store the pre-computed \(\vtheta_{leaf}\) from line 3 and \textbf{return} the node.
\EndIf
\State

\State Create an internal node, store split rule \((\vtheta_1^*, \vtheta_2^*)\)
\State \(N.\text{left\_child} \gets \texttt{BuildTree}(\calS_L, d+1, D_{max}, N_{min}, \tau_{RMSE}, \mu)\)
\State \(N.\text{right\_child} \gets \texttt{BuildTree}(\calS_R, d+1, D_{max}, N_{min}, \tau_{RMSE}, \mu)\)
\State \textbf{return} the internal node
\end{algorithmic}
\end{algorithm}

\section{Time Complexity Analysis}
\label{sec:appendix_complexity_en}

\subsection{Algorithm 1: Find Optimal Split}
\begin{itemize}
    \item \textbf{Overview:} Algorithm 1 is an iterative algorithm designed to find the optimal splitting parameters \(\vtheta_1\) and \(\vtheta_2\). It alternates between model fitting and partition updating, running for at most \(T_{\text{max}}\) iterations.
    \item \textbf{Per Iteration:}
    \begin{itemize}
        \item \textbf{Model Fitting:} Solves OLS for two partitions \(\calS_1\) and \(\calS_2\). These OLS solutions can incorporate ridge regression. The complexity of each OLS problem is \(O(N_k \cdot d^2)\), where \(N_k\) is the number of samples in the partition and \(d\) is the feature dimension. Since \(N = N_1 + N_2\), the total complexity is \(O(N \cdot d^2)\).
        \item \textbf{Partition Updating:} For each of the \(N\) samples, computes two linear functions \(\tilde{\vx}_j^T \vtheta_1\) and \(\tilde{\vx}_j^T \vtheta_2\) (each \(O(d)\)) and reassigns partitions, yielding a complexity of \(O(N \cdot d)\).
        \item \textbf{Total Per Iteration:} \(O(N \cdot d^2) + O(N \cdot d) = O(N \cdot d^2)\).
    \end{itemize}
    \item \textbf{Total Iterations:} At most \(T_{\text{max}}\), a user-defined maximum.
    \item \textbf{Total Time Complexity:} \(O(T_{\text{max}} \cdot N \cdot d^2)\).
    \item \textbf{Note:} In practice, convergence may occur before \(T_{\text{max}}\), but the worst-case complexity assumes the full number of iterations.
\end{itemize}

\subsection{Algorithm 2: Build Oblique Regression Tree}
\begin{itemize}
    \item \textbf{Overview:} Algorithm 2 recursively constructs an oblique regression tree with a maximum depth \(D_{\text{max}}\) and a minimum of \(N_{\text{min}}\) samples per leaf. Each internal node invokes Algorithm 1, while leaf nodes fit a final linear model.
    \item \textbf{Internal Nodes:}
    \begin{itemize}
        \item \textbf{Splitting:} Each internal node calls Algorithm 1, with complexity \(O(T_{\text{max}} \cdot N_{\text{node}} \cdot d^2)\), where \(N_{\text{node}}\) is the number of samples at the node.
        \item \textbf{Worst-Case Assumption:} The tree is balanced, with each split dividing the data approximately in half.
        \item \textbf{Level Analysis:} At level \(k\), there are \(2^k\) nodes, each with approximately \(N / 2^k\) samples. The total complexity per level is:
        \[
        2^k \cdot O(T_{\text{max}} \cdot (N / 2^k) \cdot d^2) = O(T_{\text{max}} \cdot N \cdot d^2).
        \]
        \item \textbf{Total Levels:} With \(D_{\text{max}}\) levels, the complexity is \(O(D_{\text{max}} \cdot T_{\text{max}} \cdot N \cdot d^2)\).
    \end{itemize}
    \item \textbf{Leaf Nodes:}
    \begin{itemize}
        \item \textbf{Fitting:} Each leaf fits an OLS model, which can also incorporate ridge regression, with complexity \(O(N_{\text{leaf}} \cdot d^2)\), where \(N_{\text{leaf}}\) is the number of samples in the leaf.
        \item \textbf{Total Leaves:} Up to \(O(2^{D_{\text{max}}})\) leaves, but the total sample count across all leaves is \(N\). Thus, the total complexity is \(O(N \cdot d^2)\).
    \end{itemize}
    \item \textbf{Total Time Complexity:} \(O(D_{\text{max}} \cdot T_{\text{max}} \cdot N \cdot d^2) + O(N \cdot d^2) = O(D_{\text{max}} \cdot T_{\text{max}} \cdot N \cdot d^2)\).
    \item \textbf{Note:} In practical settings, \(T_{\text{max}}\) and \(D_{\text{max}}\) are often fixed constants, simplifying the complexity to \(O(N \cdot d^2)\). In theory, if \(D_{\text{max}}\) scales with \(N\), the complexity increases accordingly.
\end{itemize}

\section{Ablation Study on Step Size: Stability, Efficiency, and Fallback}
\label{apendix_fallback}

The step size $\mu$ (damping factor) is a key hyperparameter in our node-splitting algorithm. It controls how aggressively each Newton update moves toward the local OLS solution and therefore directly affects stability, efficiency, and final performance. To examine these effects empirically, we performed an ablation over $\mu$ on four datasets with different characteristics: the oscillatory \texttt{sinc} function, the smoother \texttt{twisted\_sigmoid} function, and the real-world \texttt{Kinematics} and \texttt{D-Ailerons} datasets. For the two real-world datasets, we also evaluated an automatic line-search step-size strategy (denoted by \texttt{auto}). Results are reported in Tables~\ref{tab:ablation_sinc}--\ref{tab:ablation_delta_ailerons}. 

\paragraph{1. Damping and the fallback mechanism.}
On the synthetic datasets, smaller step sizes led to very few fallbacks, while more aggressive steps could trigger the fallback mechanism more often. For \texttt{sinc}, the fallback rate was essentially negligible at $\mu=0.01$ (\textbf{0.39\%}), modest at $\mu=0.05$ (\textbf{4.74\%}), and rose to \textbf{32.86\%}, \textbf{60.94\%}, and \textbf{29.55\%} at $\mu=0.10$, $0.50$, and $1.00$, respectively. A similar but milder trend appeared with \texttt{twisted\_sigmoid}: the fallback rate increased from \textbf{0.67\%} at $\mu=0.01$ to \textbf{34.67\%} at $\mu=0.50$, before dropping to \textbf{19.33\%} at $\mu=1.00$. On the real-world datasets, fallback rates were moderate and relatively stable across $\mu$ for \texttt{Kinematics} (between \textbf{6.87\%} and \textbf{10.82\%}, with \textbf{9.51\%} for \texttt{auto}), and were zero on \texttt{D-Ailerons} for all settings, including \texttt{auto}. Overall, these results indicated that damping does help keep the fallback mechanism rare on more challenging synthetic objectives, while some benign regression problems remain stable even for relatively aggressive step sizes.

\paragraph{2. Step size, iterations, and training time.}
Larger step sizes did not automatically translate into faster convergence in wall-clock time. On the \texttt{Kinematics} datasets, small or moderately damped step sizes typically led to stable and fairly regular Newton dynamics, so the node-wise optimisation terminated after only a few iterations on most splits. In contrast, aggressive step sizes could induce oscillatory or irregular behavior at some nodes: the objective value might fail to decrease monotonically, and the parameters and partitions could keep moving without settling down quickly. Such oscillatory runs required many more iterations before the stopping criterion was met, which substantially increased both the average iteration count and the total training time. As a result, the dependence of fit time on~\(\mu\) was non-monotone: large steps could be very fast when they converged cleanly, but they might be markedly less efficient once the cost of these unstable cases was taken into account.

\paragraph{3. Effect on predictive accuracy.}
The impact of $\mu$ on RMSE was dataset dependent but generally modest, with a few notable cases. On \texttt{sinc}, the smallest step size $\mu=0.01$ achieved the lowest error (\textbf{0.02796}), while $\mu=0.10$ was clearly suboptimal (\textbf{0.05310}); larger steps ($\mu=0.50$ and $1.00$) recovered competitive errors but at the cost of higher fallback rates and iteration counts. On \texttt{twisted\_sigmoid}, all choices of $\mu$ yielded similar performance, with a slight advantage around $\mu=0.50$ and $1.00$ (RMSE $\approx \textbf{0.0258}$--\textbf{0.0283}). On \texttt{Kinematics}, performance improved as $\mu$ increased up to $0.50$ (RMSE decreasing from \textbf{0.13055} at $\mu=0.01$ to \textbf{0.10266} at $\mu=0.50$), and the \texttt{auto} strategy yielded the lowest error (\textbf{0.10185}) while using far fewer iterations than $\mu=0.50$. On \texttt{D-Ailerons}, all step-size choices, including \texttt{auto}, achieved essentially identical RMSE (about \textbf{$1.68\times10^{-4}$}), indicating that accuracy was largely insensitive to the damping factor on this dataset.

\paragraph{4. Overall takeaway.}
Taken together, these experiments suggest that the damping factor $\mu$ mainly controls a trade-off between stability, efficiency, and the need for fallbacks, with only moderate influence on accuracy in most settings. Small or moderate fixed step sizes tended to keep fallbacks rare and iterations moderate, especially on synthetic problems such as \texttt{sinc} and \texttt{twisted\_sigmoid}. Extremely aggressive fixed steps could lead to increased fallback usage and substantially more iterations, particularly on more structured or high-dimensional data such as \texttt{Kinematics}. The automatic line-search step size (\texttt{auto}) consistently provided a robust choice on the real-world datasets: it attained strong accuracy, moderate iteration counts, and stable fallback behavior without requiring manual tuning of~$\mu$.

\begin{table}[htbp]
\scriptsize
\centering
\caption{Ablation study on the \texttt{sinc} dataset. Metrics are averaged over 10 runs.}
\label{tab:ablation_sinc}
\begin{tabular}{@{}cccccccc@{}} 
\toprule
Step size (\(\mu\)) & RMSE \(\downarrow\) & Leaves & Avg. Iters & Fit Time (s) & Avg. Fallbacks & Avg. Splits & Fallback Rate (\%) \\ %
\midrule
0.01 & 0.02796 & 26.7 & 2.56 & 0.166 & 0.1 & 25.7 & 0.39\% \\
0.05 & 0.02885 & 24.2 & 4.46 & 0.105 & 1.1 & 23.2 & 4.74\% \\
0.10 & 0.05310 & 22.3 & 5.97 & 0.107 & 7.0 & 21.3 & 32.86\% \\
0.50 & 0.03069 & 26.6 & 4.30 & 0.112 & 15.6 & 25.6 & 60.94\% \\
1.00 & 0.02926 & 23.0 & 6.65 & 0.141 & 6.5 & 22.0 & 29.55\% \\
\bottomrule
\end{tabular}%
\end{table}

\begin{table}[htbp]
\scriptsize
\centering
\caption{Ablation study on the \texttt{twisted\_sigmoid} dataset. Metrics are averaged over 10 runs.}
\label{tab:ablation_twisted_sigmoid}
\begin{tabular}{@{}cccccccc@{}} %
\toprule
Step size (\(\mu\)) & RMSE \(\downarrow\) & Leaves & Avg. Iters & Fit Time (s) & Avg. Fallbacks & Avg. Splits & Fallback Rate (\%) \\ %
\midrule
0.01 & 0.02653 & 16.0 & 3.17 & 0.045 & 0.1 & 15.0 & 0.67\% \\
0.05 & 0.03011 & 16.0 & 4.44 & 0.083 & 0.6 & 15.0 & 4.00\% \\
0.10 & 0.02831 & 15.6 & 5.28 & 0.063 & 3.0 & 14.6 & 20.55\% \\
0.50 & 0.02578 & 16.0 & 6.13 & 0.107 & 5.2 & 15.0 & 34.67\% \\
1.00 & 0.02577 & 16.0 & 7.54 & 0.121 & 2.9 & 15.0 & 19.33\% \\
\bottomrule
\end{tabular}%
\end{table}

\begin{table}[htbp]
\scriptsize
\centering
\caption{Ablation study on the \texttt{Kinematics} dataset. Metrics are averaged over five runs.}
\label{tab:ablation_kinematics}
\begin{tabular}{@{}cccccccc@{}} 
\toprule
Step size (\(\mu\)) & RMSE \(\downarrow\) & Leaves & Avg. Iters & Fit Time (s) & Avg. Fallbacks & Avg. Splits & Fallback Rate (\%) \\ %
\midrule
\textcolor{black}{0.01} & 0.13055 & 57.0 & \textcolor{black}{10.89} & \textcolor{black}{0.49} & 4.6 & 56.0 & 8.21\% \\
0.05 & 0.11367 & 59.2 & 19.53 & 0.62 & 4.0 & 58.2 & 6.87\% \\
0.10 & 0.10631 & 57.2 & 26.65 & 0.72 & 6.0 & 56.2 & 10.68\% \\
0.50 & 0.10266 & 56.6 & 59.77 & 1.53 & 5.8 & 57.2 & 10.14\% \\
\textcolor{black}{1.00} & 0.10702 & 48.6 & \textcolor{black}{83.43} & \textcolor{black}{1.88} & 5.8 & 53.6 & 10.82\% \\
\textcolor{black}{auto} & 0.10185 & 48.6 & \textcolor{black}{15.76} & \textcolor{black}{0.79} & 5.4 & 56.8 & 9.51\% \\
\bottomrule
\end{tabular}%
\end{table}

\begin{table}[htbp]
\scriptsize
\centering
\caption{Ablation study on the \texttt{D-Ailerons} dataset. Metrics are averaged over five runs. The reported RMSE values are scaled by \(\times 10^{-4}\).}
\label{tab:ablation_delta_ailerons} 
\begin{tabular}{@{}cccccccc@{}} 
\toprule
Step size (\(\mu\)) & RMSE \(\downarrow\) & Leaves & Avg. Iters & Fit Time (s) & Avg. Fallbacks & Avg. Splits & Fallback Rate (\%) \\ 
\midrule
0.01 & 1.68 & 8 & 1.57 & 0.105 & 0 & 7 & 0\% \\
0.05 & 1.68 & 8 & 32.43 & 0.178 & 0 & 7 & 0\% \\
0.10 & 1.68 & 8 & 30.09 & 0.179 & 0 & 7 & 0\% \\
0.50 & 1.68 & 8 & 57.69 & 0.345 & 0& 7 & 0\% \\
1.00 & 1.68 & 8 & 92.60 & 0.525 & 0 & 7 & 0\% \\
auto & 1.68 & 8 & 6.71 & 0.183 & 0 & 7 & 0\% \\
\bottomrule
\end{tabular}%
\end{table}

\clearpage
\newpage

\section{Synthetic Data Description}
\label{sec:appendix_synthetic_functions}

This appendix provides the detailed mathematical definitions of the 3D synthetic functions, the rationale for the selection of these synthetic functions and noise levels, and additional visualizations of our method's 2D and 3D function approximation capabilities (Figures \ref{fig:model_comparison} and \ref{fig:3d_model_comparison}).

For all 3D functions, the inputs \(x_1, x_2\) range from \([-3, 3]\).

\paragraph{3D Oscillatory Surface Functions:}
\begin{itemize}
    \item \( f_1(x_1,x_2)=\left( 0.5\, x_{1}^{3} - 2\, x_{1} x_{2}^{2} + 3 \sin\left( 4 x_{1} \right) \cos\left( 2 x_{2} \right) + 0.1\, e^{-\left( x_{1}^{2} + x_{2}^{2} \right)} \right) \)
    \item \( f_2(x_1,x_2)=\sin(3x_{1}) + \cos(2x_{2}) + 0.5 \, \sin(5x_{1}) \cos(4x_{2}) \)
    \item \( f_3(x_1, x_2) = \frac{x_1^2 - x_2^2}{0.5 + r^2} + \sin(r) e^{-r} \), where \(r = \sqrt{x_1^2 + x_2^2} + 10^{-6}\).
    \item \( f_4(x_1, x_2) = 2 \cdot \exp\!\left(-\frac{(x_1 - 1)^2 + (x_2 - 1)^2}{0.5}\right) - 3 \cdot \exp\!\left(-\frac{(x_1 + 1)^2 + (x_2 + 1.5)^2}{0.3}\right) + 0.5 \, x_1 \)
\end{itemize}

\textbf{Rationale for Selection of Synthetic Functions:}
 These functions exhibited diverse nonlinear characteristics, including high-frequency oscillations (\texttt{sinc}, 3D oscillatory functions), sharp changes, and inflection points (\texttt{twisted\_sigmoid}). Successfully approximating such a wide range of nonlinearities, as demonstrated by our empirical results, provides strong empirical support for the universal approximation capability of our piecewise linear HRT model.
 The \texttt{sinc} function, with its numerous local extrema and oscillations, presents a particularly challenging landscape for optimization algorithms. It helps to expose potential instabilities (e.g., non-monotonic convergence, limit cycles) in the splitting mechanism, thereby highlighting the necessity and effectiveness of damped Newton updates for robust convergence. The \texttt{twisted\_sigmoid} function, while non-trivial, is smoother and less prone to pathological behaviors, allowing us to demonstrate the efficiency of our Newton updates in a more stable environment.

\textbf{Rationale for Noise Levels:}
Introducing independent and identically distributed zero-mean Gaussian noise (with standard deviations of \(0.025\) for 2D tasks and \(0.05\) for 3D tasks) simulated inherent uncertainty. This ensured that our model's performance was evaluated not just on ideal, noiseless functions, but on data that more closely mimic practical applications.
The selected noise levels were not set so high that the primary challenge for the models remained the accurate approximation of the underlying functional relationship, rather than primarily having to distinguish signal from overwhelming noise. This allowed us to clearly observe and compare the models' capacities to fit the true, underlying nonlinear structures.

\begin{figure*}[htbp]
    \centering
    \begin{subfigure}{0.48\textwidth}
        \centering
        \includegraphics[width=\textwidth]{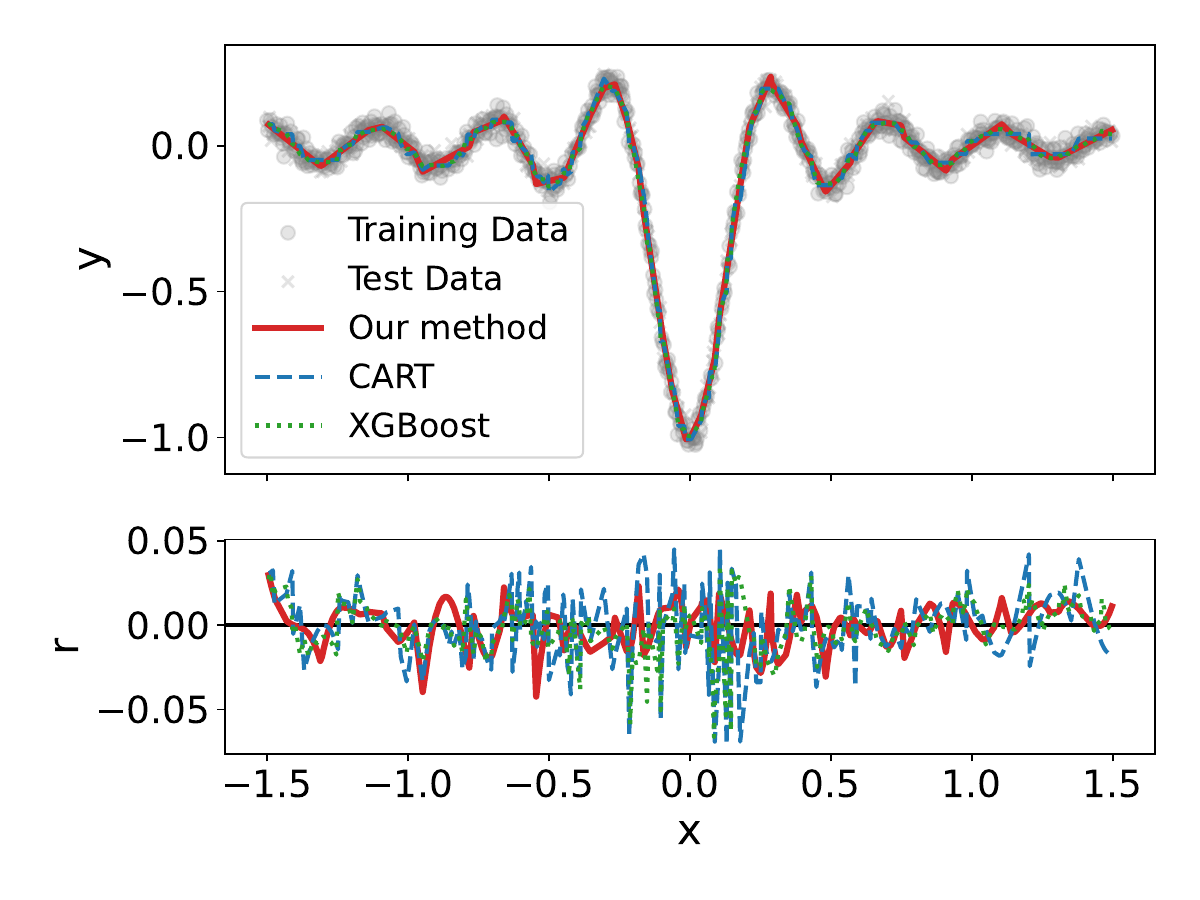}
        \caption{ \(y = -\frac{\sin(5\pi x)}{5\pi x} + 0.025\epsilon\)}
        \label{sinc}
    \end{subfigure}
    \hfill
    \begin{subfigure}{0.48\textwidth}
        \centering
        \includegraphics[width=\textwidth]{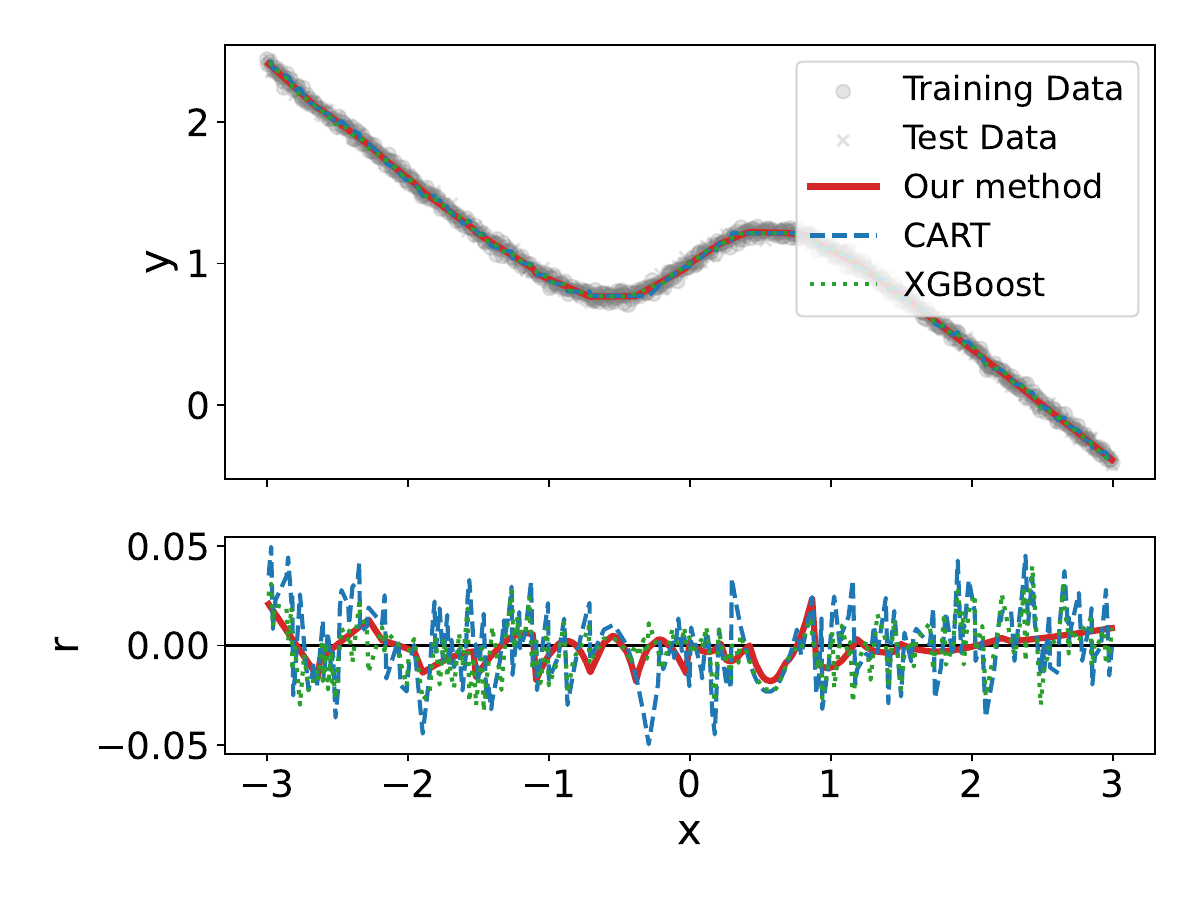}
        \caption{  \(y = \frac{2}{1+e^{-3x}} - 0.8x + 0.025\epsilon\)}
        \label{twist}
    \end{subfigure}
    \caption{Top: Performance of various methods approximating \texttt{sinc} and \texttt{twisted\_sigmoid} functions. 
    Bottom: Residuals \( r = y_{\text{pred}} - y_{\text{true}} \), representing the difference between the predicted and true values for each method.}
    \label{fig:model_comparison}
\end{figure*}

\begin{figure*}[t]
\vskip -0.1in
    \centering
    \begin{subfigure}{0.48\textwidth}
        \centering
        \includegraphics[width=\textwidth]{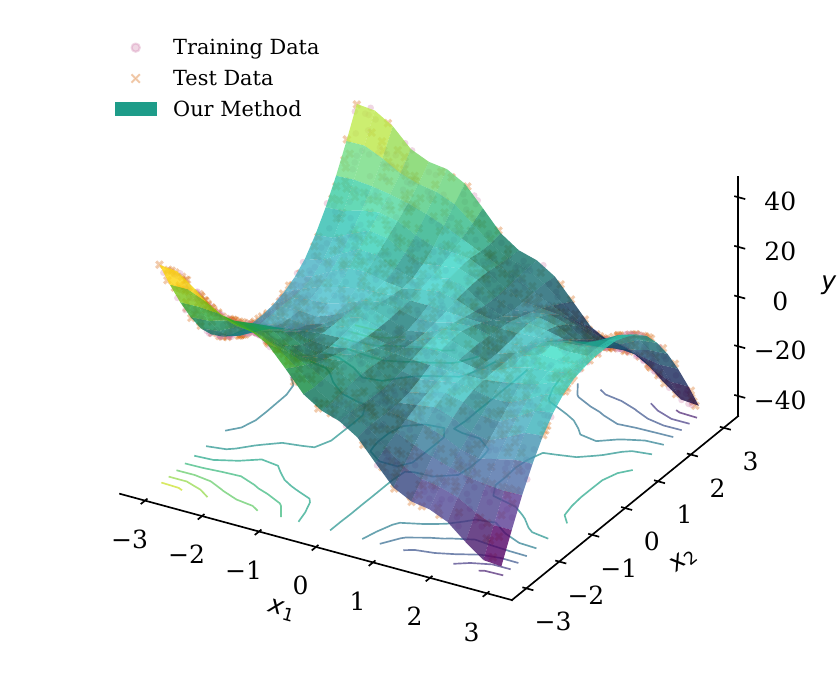}
        \caption{ Our Method approximating \( y=f_1(x_1,x_2)+0.05\epsilon\).}
    \end{subfigure}
    \hfill
    \begin{subfigure}{0.48\textwidth}
        \centering
        \includegraphics[width=\textwidth]{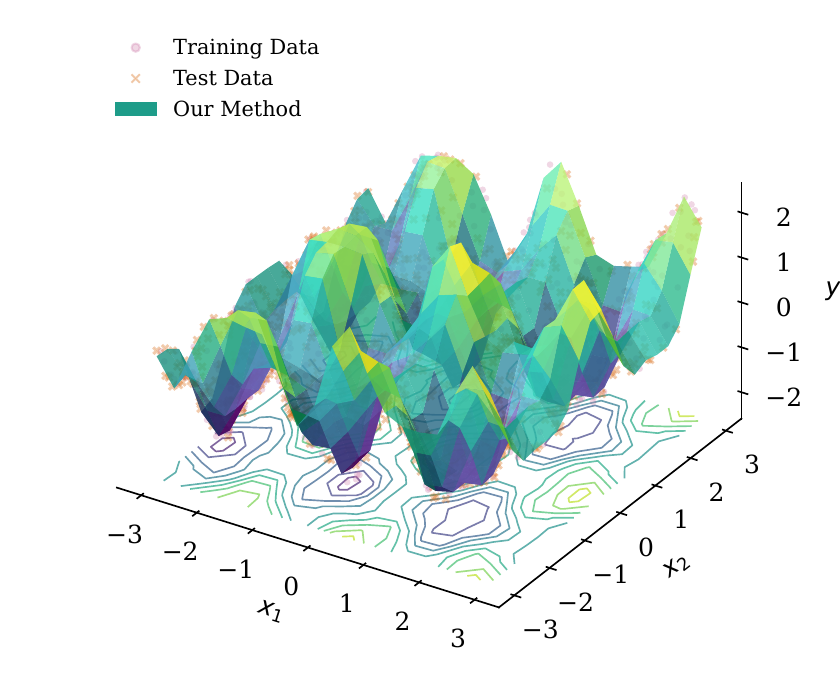}
        \caption{Our Method approximating \( y=f_2(x_1,x_2)+0.05\epsilon\).}
    \end{subfigure}
      \hfill
    \begin{subfigure}{0.48\textwidth}
        \centering
        \includegraphics[width=\textwidth]{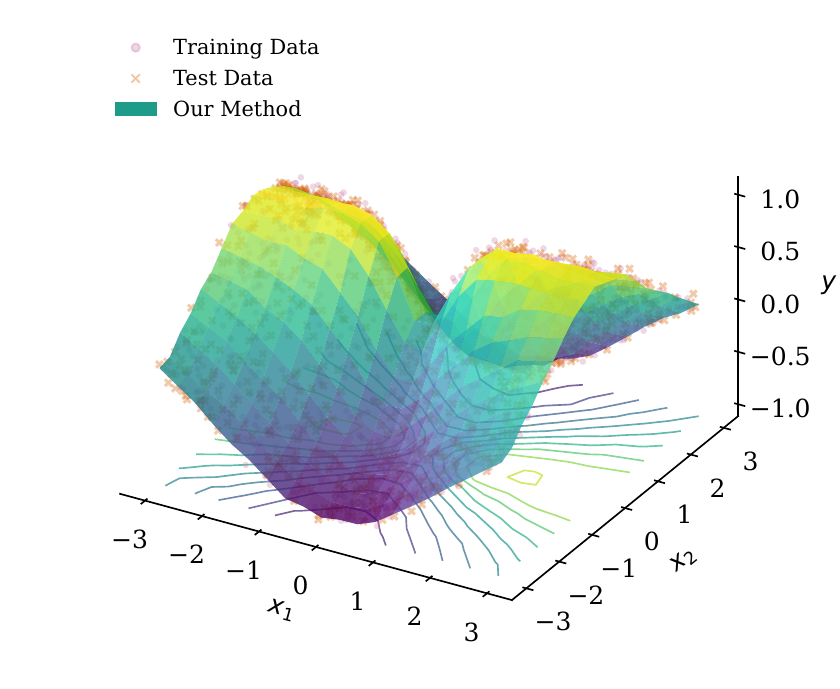}
        \caption{ Our Method approximating \( y=f_3(x_1,x_2)+0.05\epsilon\).}
        \label{f3_plot_appendix}
    \end{subfigure}
    \hfill
    \begin{subfigure}{0.48\textwidth}
        \centering
        \includegraphics[width=\textwidth]{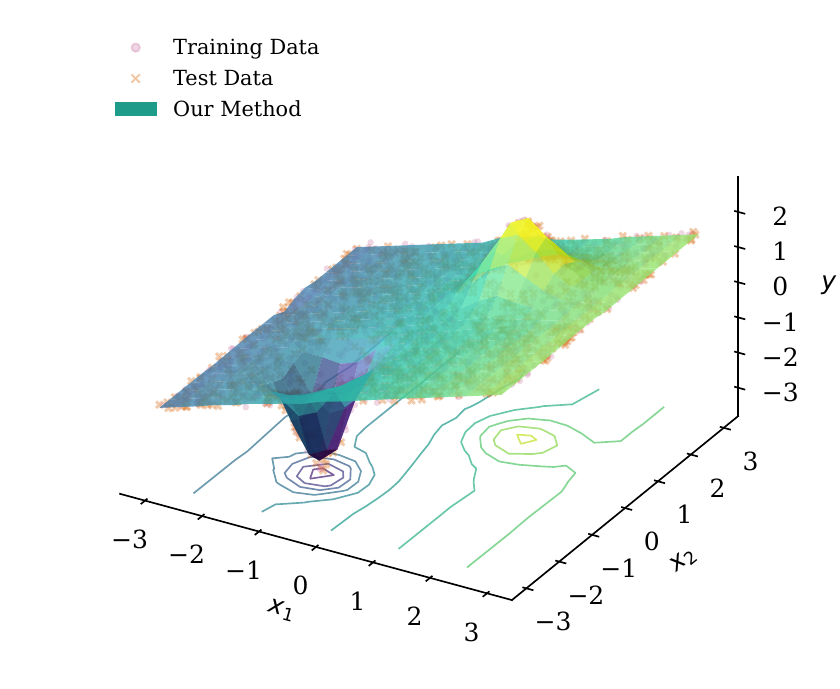}
        \caption{Our Method approximating \( y=f_4(x_1,x_2)+0.05\epsilon\).}
        \label{f4_plot_appendix}
    \end{subfigure}
    \caption{Approximation of 3D functions.
    We visualize the learned piecewise linear surface of our method along with training/testing points. Contours on the floor help read the depth. For clarity, only the fitted surface of our Method is shown; detailed quantitative comparisons with baseline methods are provided in the main text.
    }
    \label{fig:3d_model_comparison}.
    \vskip -0.3in
\end{figure*}

\clearpage
\newpage

\section{Hyperparameters}
\label{sec:appendix_hyperparameters}

This appendix details the hyperparameter search grids and tuning procedures used for both the 2D and 3D function approximation experiments and the real-world regression tasks. For each model, a grid search approach was employed to find the optimal hyperparameters that minimize the validation error. The specific search ranges for the baseline 3D approximation models are listed in Table \ref{tab:hyperparameter_grids}. Hyperparameters for all models were tuned via five-fold cross-validation on the training sets, with the final performance reported on the testing sets.

\begin{table}[!ht]
\small
    \centering
    \caption{Hyperparameter Search Grids for 3D Function Approximation Models}
    \label{tab:hyperparameter_grids}
    \begin{tabular}{lll}
        \toprule
        Model & Hyperparameter & Search Grid \\
        \midrule
        \multirow{4}{*}{HRT} 
                                    & \texttt{max\_depth} & \( \{4, 6, 8, 10, 12\} \) \\
                                    & \texttt{threshold} & \( \{0.01, 0.03, 0.05\} \) \\
                                    & \texttt{step\_size} & \( \{0.01, 0.5, 1, \texttt{'auto'}\} \)\\
                                    & \texttt{ridge\_alpha} & \( \{0.001, 0, 1\} \) \\
        \midrule
        \multirow{3}{*}{CART}       & \texttt{max\_depth} & \( \{3, 5, 7, 9, 11\} \) \\
                                    & \texttt{min\_samples\_split} & \( \{2, 5, 10\} \) \\
                                    & \texttt{min\_samples\_leaf} & \( \{1, 2, 4\} \) \\
        \midrule
        \multirow{4}{*}{XGBoost}    & \texttt{n\_estimators} & \( \{50, 100, 200\} \) \\
                                    & \texttt{learning\_rate} & \( \{0.01, 0.1, 0.2\} \) \\
                                    & \texttt{max\_depth} & \( \{2, 3, 5\} \) \\
                                    & \texttt{subsample} & \( \{0.7, 1.0\} \) \\
        \bottomrule
    \end{tabular}
\end{table}

For the single-tree models, we carefully constrained the search space to ensure fair complexity comparisons. For HRT, we tuned the \texttt{threshold} parameter, which controls the splitting criterion, and \texttt{max\_depth}, which limits the maximum depth of the constructed tree structure. We also varied the \texttt{ridge\_alpha} parameter, which adjusts the regularization strength in ridge regression to mitigate overfitting, and the \texttt{step\_size}, which determines the granularity of updates during optimization and affects convergence speed and stability. For CART, we explored variations in \texttt{max\_depth}, the minimum number of samples required to split an internal node (\texttt{min\_samples\_split}), and the minimum number of samples required to be at a leaf node (\texttt{min\_samples\_leaf}). 

For the ensemble models evaluated on the real-world datasets, we expanded our grid search to thoroughly benchmark HRT-Boost against highly optimized baseline methods. For our proposed HRT-Boost, we tuned the number of estimators (\texttt{n\_estimators}), maximum depth (\texttt{max\_depth}), learning rate (\texttt{learning\_rate}), and the ridge penalty (\texttt{ridge\_alpha}). For Random Forest (RF), we optimized the maximum depth, the number of estimators, the minimum samples per leaf, and the feature subsampling strategy (\texttt{max\_features}). For AdaBoost, we searched over the learning rate, the loss function (\texttt{loss} $\in \{\texttt{'linear'}, \texttt{'square'}\}$), and the number of estimators. For Scikit-GBM, we tuned the learning rate, maximum depth, number of estimators, and subsampling ratio. For XGBoost, we optimized the number of boosting rounds (\texttt{n\_estimators}), the step size shrinkage (\texttt{learning\_rate}), the maximum depth of a tree (\texttt{max\_depth}), and the subsample ratio of the training instances (\texttt{subsample}). For LightGBM, our grid included the learning rate, number of estimators, maximum leaves per tree (\texttt{num\_leaves}), and L1 regularization (\texttt{reg\_alpha}). 

Tables~\ref{tab:all_datasets_FHyperparameters} and \ref{tab:all_datasets_Hyperparameters} summarize the chosen hyperparameters via 5-fold cross-validation for the function approximation tasks and single-tree real-world benchmarks, respectively. The tuned hyperparameters for the neural network baselines (DTSemNet and DGT) are listed in Table~\ref{tab:hyperparameters of DGT and DTSemNet}, and optimal configurations for all ensemble methods across the diverse regression datasets are detailed in Table~\ref{tab:all_datasets_Hyperparameters2} and Table~\ref{tab:optimal_hyperparameters}.

\begin{table*}[htbp] 
\tiny
    \centering
    \caption{Optimal Hyperparameters Across Various Functions for Different Models} 
    \label{tab:all_datasets_FHyperparameters}
    \begin{tabular}{l l l} 
        \toprule
        \textbf{Function} & \textbf{Model} & \textbf{Hyperparameters} \\
        \midrule 
        \multirow{3}{*}{sinc} 
        & CART & \texttt{\{'max\_depth': 9, 'min\_samples\_leaf': 2, 'min\_samples\_split': 2\}} \\
        & XGB & \texttt{\{'learning\_rate': 0.1, 'max\_depth': 3, 'n\_estimators': 200, 'subsample': 0.7\}} \\
        & HRT & \texttt{\{'max\_depth': 6, 'ridge\_alpha': 0.001, 'step\_size': 0.01, 'threshold': 0.03\}} \\
                \midrule 
        \multirow{3}{*}{twisted\_sigmoid} 
        & CART & \texttt{\{'max\_depth': 9, 'min\_samples\_leaf': 2, 'min\_samples\_split': 5\}} \\
        & XGB & \texttt{\{'learning\_rate': 0.1, 'max\_depth': 3, 'n\_estimators': 100, 'subsample': 0.7\}} \\
        & HRT & \texttt{\{'max\_depth': 4, 'ridge\_alpha': 0.001, 'step\_size': 0.5, 'threshold': 0.01\}} \\
                \midrule 
        \multirow{3}{*}{$f_1(x_1,x_2)$} 
        & CART & \texttt{\{'max\_depth': 11, 'min\_samples\_leaf': 1, 'min\_samples\_split': 2\}} \\
        & XGB & \texttt{\{'learning\_rate': 0.2, 'max\_depth': 5, 'n\_estimators': 200, 'subsample': 0.7\}} \\
        & HRT & \texttt{\{'max\_depth': 12, 'threshold': 0.01, 'ridge\_alpha': 0, 'step\_size': 1\}} \\
                \midrule 
                \multirow{3}{*}{ $f_2(x_1,x_2)$} 
        & CART & \texttt{\{'max\_depth': 11, 'min\_samples\_leaf': 1, 'min\_samples\_split': 2\}} \\
        & XGB & \texttt{\{'learning\_rate': 0.2, 'max\_depth': 5, 'n\_estimators': 200, 'subsample': 0.7\}} \\
        & HRT & \texttt{\{'max\_depth': 12, 'threshold': 0.01, 'ridge\_alpha': 0, 'step\_size': 1\}} \\
                        \midrule 
        \multirow{3}{*}{$f_3(x_1,x_2)$} 
        & CART & \texttt{\{'max\_depth': 11, 'min\_samples\_leaf': 4, 'min\_samples\_split': 2\}} \\
        & XGB & \texttt{\{'learning\_rate': 0.1, 'max\_depth': 5, 'n\_estimators': 200, 'subsample': 0.7\}} \\
        & HRT & \texttt{\{'max\_depth': 8, 'threshold': 0.05, 'ridge\_alpha': 0, 'step\_size': 1\}} \\
                                \midrule 
        \multirow{3}{*}{$f_4(x_1,x_2)$} 
        & CART & \texttt{\{'max\_depth': 11, 'min\_samples\_leaf': 1, 'min\_samples\_split': 2\}} \\
        & XGB & \texttt{\{'learning\_rate': 0.1, 'max\_depth': 5, 'n\_estimators': 200, 'subsample': 0.7\}} \\
        & HRT & \texttt{\{'max\_depth': 12, 'threshold': 0.05, 'ridge\_alpha': 0, 'step\_size': 1\}} \\
        \bottomrule
    \end{tabular}
\end{table*}

\begin{table*}[htbp] 
\tiny
    \centering
    \caption{Optimal Hyperparameters Across Various Datasets for Different Models} 
    \label{tab:all_datasets_Hyperparameters}
    \begin{tabular}{l l l} 
        \toprule
        \textbf{Dataset } & \textbf{Model} & \textbf{Hyperparameters} \\
        \midrule 
        \multirow{3}{*}{ Abalone } 
        & M5 & \texttt{\{'M': 10.0\}} \\
        & Linear tree & \texttt{\{'max\_depth': 3, 'min\_samples\_leaf': 40\}} \\
        & HRT & \texttt{\{'max\_depth': 2, 'ridge\_alpha': 0, 'step\_size': 1, 'threshold': 1\}} \\
                \midrule 
        \multirow{3}{*}{CPUact } 
        & M5 & \texttt{\{'M': 4.0\}} \\
        & Linear tree & \texttt{\{'max\_depth': 3, 'min\_samples\_leaf': 10\}} \\
        & HRT & \texttt{\{'max\_depth': 2, 'ridge\_alpha': 0, 'step\_size': 0.5, 'threshold': 0.5\}} \\
                \midrule 
        \multirow{3}{*}{Ailerons } 
        & M5 & \texttt{\{'M': 40.0 \}} \\
        & Linear tree & \texttt{\{'max\_depth': 3, 'min\_samples\_leaf': 10\}} \\
        & HRT & \texttt{\{'max\_depth': 2, 'ridge\_alpha': 1, 'step\_size': 1, 'threshold': 5e-05\}} \\
         \midrule 
            \multirow{3}{*}{CTSlice } 
                & M5 & \texttt{\{'M': 20.0\}} \\
        & Linear tree & \texttt{\{'max\_depth': 5, 'min\_samples\_leaf': 10\}} \\
        & HRT & \texttt{\{'max\_depth': 7, 'ridge\_alpha': 10, 'step\_size': 1, 'threshold': 0.2\}} \\
                \midrule 
        \multirow{3}{*}{YearPred } 
        & M5 & \texttt{\{'M': 200.0\}} \\
        & Linear tree & \texttt{\{ max\_depth=5, min\_samples\_leaf=2000 \}} \\
        & HRT & \texttt{\{'max\_depth': 4, 'ridge\_alpha': 0, 'step\_size': ‘auto’, 'threshold': 0.5\}} \\
                \midrule 
        \multirow{3}{*}{Concrete } 
        & M5 & \texttt{\{'M': 4.0\}} \\
        & Linear tree & \texttt{\{'max\_depth': 3, 'min\_samples\_leaf': 40\}} \\
        & HRT & \texttt{\{'max\_depth': 3, 'ridge\_alpha': 0.1, 'step\_size': 0.5, 'threshold': 6\}} \\
                \midrule 
        \multirow{3}{*}{Airfoil } 
        & M5 & \texttt{\{ M': 4.0\}} \\
        & Linear tree & \texttt{\{max\_depth': 6, 'min\_samples\_leaf': 50 \}} \\
        & HRT & \texttt{\{'max\_depth': 5, 'ridge\_alpha': 0.01, 'step\_size': ‘auto’, 'threshold': 1.5\}} \\
                \midrule 
        \multirow{4}{*}{Fried } 
        & CART & \texttt{\{'max\_depth': 11, 'min\_samples\_leaf': 4, 'min\_samples\_split': 10\}} \\
        & M5 & \texttt{\{'M': 4.0\}} \\
        & Linear tree & \texttt{\{'max\_depth': 5, 'min\_samples\_leaf': 10\}} \\
        & HRT & \texttt{\{'max\_depth': 5, 'ridge\_alpha': 0.1, 'step\_size': 0.1, 'threshold': 0\}} \\
                                \midrule 
        \multirow{4}{*}{D-Elevators } 
        & CART & \texttt{\{'max\_depth': 5, 'min\_samples\_leaf': 2, 'min\_samples\_split': 10\}} \\
        & M5 & \texttt{\{'M': 20.0\}} \\
        & Linear tree & \texttt{\{'max\_depth': 5, 'min\_samples\_leaf': 20\}} \\
        & HRT & \texttt{\{'max\_depth': 5, 'ridge\_alpha': 10, 'step\_size': 1, 'threshold': 0\}} \\
                                \midrule 
        \multirow{4}{*}{D-Ailerons } 
        & CART & \texttt{\{'max\_depth': 5, 'min\_samples\_leaf': 4, 'min\_samples\_split': 10\}} \\
        & M5 & \texttt{\{'M': 20.0\}} \\
        & Linear tree & \texttt{\{'max\_depth': 3, 'min\_samples\_leaf': 20\}} \\
        & HRT & \texttt{\{'max\_depth': 3, 'ridge\_alpha': 10, 'step\_size': 'auto', 'threshold': 0\}} \\
                                \midrule 
        \multirow{4}{*}{Kinematics } 
        & CART & \texttt{\{'max\_depth': 9, 'min\_samples\_leaf': 4, 'min\_samples\_split': 10\}} \\
        & M5 & \texttt{\{'M': 4.0\}} \\
        & Linear tree & \texttt{\{'max\_depth': 7, 'min\_samples\_leaf': 40\}} \\
        & HRT & \texttt{\{'max\_depth': 6, 'ridge\_alpha': 1, 'step\_size': 'auto', 'threshold': 0\}} \\
                                        \midrule 
        \multirow{4}{*}{C\&C } 
        & CART & \texttt{\{'max\_depth': 3, 'min\_samples\_leaf': 4, 'min\_samples\_split': 2\}} \\
        & M5 & \texttt{\{'M': 40.0\}} \\
        & Linear tree & \texttt{\{'max\_depth': 9, 'min\_samples\_leaf': 10\}} \\
        & HRT & \texttt{\{'max\_depth': 1, 'ridge\_alpha': 300, 'step\_size': 0.01, 'threshold': 0.0\}} \\
        \bottomrule
    \end{tabular}
\end{table*}

\begin{table}[htbp]
\small
	\centering
        \caption{Hyperparameters for DGT and DTSemNet}
         \label{tab:hyperparameters of DGT and DTSemNet}  
         \begin{tabular}{lcccccccc} 
         \toprule
                Parameters & DTSemNet & DGT \\ 
         \midrule
            Height & 5 & 6 \\
            Learning Rate & 0.005 & 0.01\\
            Batch Size & 32 & 128\\
            Num Epochs & 80 & 200\\
            Optimizer & Adam & RMSprop\\
         \bottomrule
         \end{tabular}
\end{table}

\begin{table*}[htbp] 
\scriptsize
    \centering
    \caption{Optimal Hyperparameters Across Various Datasets for Different Models} 
    \label{tab:all_datasets_Hyperparameters2}\resizebox{\textwidth}{!}{
    \begin{tabular}{l l l} 
        \toprule
        \textbf{Dataset } & \textbf{Model} & \textbf{Hyperparameters} \\
        \midrule 
        \multirow{6}{*}{ Abalone } 
        & HRT-Boost & \texttt{{'learning\_rate': 0.1, 'max\_depth': 2, 'n\_estimators': 50, 'ridge\_alpha': 1.0}} \\
        & RF & \texttt{{'max\_depth': 7, 'max\_features': None, 'min\_samples\_leaf': 4, 'n\_estimators': 150}} \\
        & AdaBoost & \texttt{{'learning\_rate': 0.1, 'loss': 'linear', 'n\_estimators': 50}} \\
        & Scikit-GBM & \texttt{{'learning\_rate': 0.1, 'max\_depth': 4, 'n\_estimators': 50, 'subsample': 0.8}} \\
        & XGBoost & \texttt{{'colsample\_bytree': 0.8, 'learning\_rate': 0.1, 'max\_depth': 3, 'n\_estimators': 150, 'subsample': 0.8}} \\
        & LightGBM & \texttt{{'learning\_rate': 0.1, 'n\_estimators': 50, 'num\_leaves': 31, 'reg\_alpha': 0.1}} \\
        \midrule 
        \multirow{6}{*}{ CPUact } 
        & HRT-Boost & \texttt{{'learning\_rate': 0.1, 'max\_depth': 4, 'n\_estimators': 150, 'ridge\_alpha': 1.0}} \\
        & RF & \texttt{{'max\_depth': 7, 'max\_features': None, 'min\_samples\_leaf': 4, 'n\_estimators': 150}} \\
        & AdaBoost & \texttt{{'learning\_rate': 0.1, 'loss': 'linear', 'n\_estimators': 150}} \\
        & Scikit-GBM & \texttt{{'learning\_rate': 0.1, 'max\_depth': 4, 'n\_estimators': 150, 'subsample': 0.8}} \\
        & XGBoost & \texttt{{'colsample\_bytree': 0.8, 'learning\_rate': 0.1, 'max\_depth': 6, 'n\_estimators': 150, 'subsample': 0.8}} \\
        & LightGBM & \texttt{{'learning\_rate': 0.1, 'n\_estimators': 150, 'num\_leaves': 31, 'reg\_alpha': 0.1}} \\
        \midrule 
        \multirow{6}{*}{ Ailerons } 
        & HRT-Boost & \texttt{{'learning\_rate': 0.1, 'max\_depth': 3, 'n\_estimators': 50, 'ridge\_alpha': 1.0}} \\
        & RF & \texttt{{'max\_depth': 7, 'max\_features': None, 'min\_samples\_leaf': 4, 'n\_estimators': 150}} \\
        & AdaBoost & \texttt{{'learning\_rate': 0.5, 'loss': 'linear', 'n\_estimators': 50}} \\
        & Scikit-GBM & \texttt{{'learning\_rate': 0.1, 'max\_depth': 4, 'n\_estimators': 150, 'subsample': 0.8}} \\
        & XGBoost & \texttt{{'colsample\_bytree': 0.8, 'learning\_rate': 0.1, 'max\_depth': 6, 'n\_estimators': 150, 'subsample': 0.8}} \\
        & LightGBM & \texttt{{'learning\_rate': 0.1, 'n\_estimators': 150, 'num\_leaves': 31, 'reg\_alpha': 0.0}} \\
        \midrule 
        \multirow{5}{*}{ CTSlice } 
        & HRT-Boost & \texttt{{'learning\_rate': 0.1, 'max\_depth': 6, 'n\_estimators': 50, 'ridge\_alpha': 10}} \\
        & XGBoost & \texttt{{'colsample\_bytree': 0.8, 'learning\_rate': 0.1, 'max\_depth': 6, 'n\_estimators': 150, 'subsample': 0.8}} \\
        & RF & \texttt{{'max\_depth': 5, 'max\_features': None, 'min\_samples\_leaf': 4, 'n\_estimators': 150}} \\
        & AdaBoost & \texttt{{'learning\_rate': 0.5, 'loss': 'square', 'n\_estimators': 150}} \\
        & Scikit-GBM & \texttt{{'learning\_rate': 0.1, 'max\_depth': 4, 'n\_estimators': 150, 'subsample': 1.0}} \\
        & LightGBM & \texttt{{'learning\_rate': 0.1, 'n\_estimators': 150, 'num\_leaves': 63, 'reg\_alpha': 0.0}} \\
        \midrule 
        \multirow{6}{*}{ YearPred } 
        & HRT-Boost & \texttt{{'learning\_rate': 0.1, 'max\_depth': 4, 'n\_estimators': 50, 'ridge\_alpha': 0}} \\
        & RF & \texttt{{'max\_depth': 5, 'max\_features': None, 'min\_samples\_leaf': 4, 'n\_estimators': 150}} \\
        & AdaBoost & \texttt{{'learning\_rate': 0.5, 'loss': 'linear', 'n\_estimators': 150}} \\
        & Scikit-GBM & \texttt{{'learning\_rate': 0.1, 'max\_depth': 4, 'n\_estimators': 150, 'subsample': 1.0}} \\
        & XGBoost & \texttt{{'colsample\_bytree': 1.0, 'learning\_rate': 0.1, 'max\_depth': 6, 'n\_estimators': 150, 'subsample': 1.0}} \\
        & LightGBM & \texttt{{'learning\_rate': 0.1, 'n\_estimators': 150, 'num\_leaves': 63, 'reg\_alpha': 0.0}} \\
        \midrule 
        \multirow{6}{*}{ Concrete } 
        & HRT-Boost & \texttt{{'learning\_rate': 0.1, 'max\_depth': 4, 'n\_estimators': 150, 'ridge\_alpha': 1.0}} \\
        & RF & \texttt{{'max\_depth': 7, 'max\_features': None, 'min\_samples\_leaf': 1, 'n\_estimators': 150}} \\
        & AdaBoost & \texttt{{'learning\_rate': 1.0, 'loss': 'square', 'n\_estimators': 150}} \\
        & Scikit-GBM & \texttt{{'learning\_rate': 0.1, 'max\_depth': 4, 'n\_estimators': 150, 'subsample': 0.8}} \\
        & XGBoost & \texttt{{'colsample\_bytree': 0.8, 'learning\_rate': 0.1, 'max\_depth': 6, 'n\_estimators': 150, 'subsample': 0.8}} \\
        & LightGBM & \texttt{{'learning\_rate': 0.1, 'n\_estimators': 150, 'num\_leaves': 31, 'reg\_alpha': 0.0}} \\
        \midrule 
        \multirow{6}{*}{ Airfoil } 
        & HRT-Boost & \texttt{{'learning\_rate': 0.1, 'max\_depth': 4, 'n\_estimators': 150, 'ridge\_alpha': 0.1}} \\
        & RF & \texttt{{'max\_depth': 7, 'max\_features': None, 'min\_samples\_leaf': 1, 'n\_estimators': 150}} \\
        & AdaBoost & \texttt{{'learning\_rate': 1.0, 'loss': 'square', 'n\_estimators': 150}} \\
        & Scikit-GBM & \texttt{{'learning\_rate': 0.1, 'max\_depth': 4, 'n\_estimators': 150, 'subsample': 0.8}} \\
        & XGBoost & \texttt{{'colsample\_bytree': 1.0, 'learning\_rate': 0.1, 'max\_depth': 6, 'n\_estimators': 150, 'subsample': 0.8}} \\
        & LightGBM & \texttt{{'learning\_rate': 0.1, 'n\_estimators': 150, 'num\_leaves': 31, 'reg\_alpha': 0.1}} \\
        \midrule 
        \multirow{6}{*}{ Fried } 
        & HRT-Boost & \texttt{{'learning\_rate': 0.1, 'max\_depth': 3, 'n\_estimators': 150, 'ridge\_alpha': 1.0}} \\
        & RF & \texttt{{'max\_depth': 7, 'max\_features': None, 'min\_samples\_leaf': 4, 'n\_estimators': 150}} \\
        & AdaBoost & \texttt{{'learning\_rate': 1.0, 'loss': 'square', 'n\_estimators': 150}} \\
        & Scikit-GBM & \texttt{{'learning\_rate': 0.1, 'max\_depth': 4, 'n\_estimators': 150, 'subsample': 0.8}} \\
        & XGBoost & \texttt{{'colsample\_bytree': 0.8, 'learning\_rate': 0.1, 'max\_depth': 6, 'n\_estimators': 150, 'subsample': 0.8}} \\
        & LightGBM & \texttt{{'learning\_rate': 0.1, 'n\_estimators': 150, 'num\_leaves': 31, 'reg\_alpha': 0.0}} \\
        \midrule 
        \multirow{6}{*}{  D-Elevators } 
        & HRT-Boost & \texttt{{'learning\_rate': 0.1, 'max\_depth': 4, 'n\_estimators': 50, 'ridge\_alpha': 1.0}} \\
        & RF & \texttt{{'max\_depth': 7, 'max\_features': None, 'min\_samples\_leaf': 1, 'n\_estimators': 150}} \\
        & AdaBoost & \texttt{{'learning\_rate': 0.1, 'loss': 'linear', 'n\_estimators': 50}} \\
        & Scikit-GBM & \texttt{{'learning\_rate': 0.1, 'max\_depth': 4, 'n\_estimators': 50, 'subsample': 0.8}} \\
        & XGBoost & \texttt{{'colsample\_bytree': 0.8, 'learning\_rate': 0.1, 'max\_depth': 3, 'n\_estimators': 150, 'subsample': 1.0}} \\
        & LightGBM & \texttt{{'learning\_rate': 0.1, 'n\_estimators': 50, 'num\_leaves': 31, 'reg\_alpha': 0.0}} \\
        \midrule 
        \multirow{6}{*}{ D-Ailerons } 
        & HRT-Boost & \texttt{{'learning\_rate': 0.1, 'max\_depth': 2, 'n\_estimators': 150, 'ridge\_alpha': 1.0}} \\
        & RF & \texttt{{'max\_depth': 7, 'max\_features': None, 'min\_samples\_leaf': 4, 'n\_estimators': 150}} \\
        & AdaBoost & \texttt{{'learning\_rate': 0.1, 'loss': 'linear', 'n\_estimators': 50}} \\
        & Scikit-GBM & \texttt{{'learning\_rate': 0.1, 'max\_depth': 4, 'n\_estimators': 50, 'subsample': 0.8}} \\
        & XGBoost & \texttt{{'colsample\_bytree': 0.8, 'learning\_rate': 0.1, 'max\_depth': 3, 'n\_estimators': 150, 'subsample': 0.8}} \\
        & LightGBM & \texttt{{'learning\_rate': 0.1, 'n\_estimators': 50, 'num\_leaves': 31, 'reg\_alpha': 0.0}} \\
        \midrule 
        \multirow{6}{*}{ Kinematics } 
        & HRT-Boost & \texttt{{'learning\_rate': 0.1, 'max\_depth': 4, 'n\_estimators': 150, 'ridge\_alpha': 1.0}} \\
        & RF & \texttt{{'max\_depth': 7, 'max\_features': None, 'min\_samples\_leaf': 4, 'n\_estimators': 150}} \\
        & AdaBoost & \texttt{{'learning\_rate': 1.0, 'loss': 'square', 'n\_estimators': 150}} \\
        & Scikit-GBM & \texttt{{'learning\_rate': 0.1, 'max\_depth': 4, 'n\_estimators': 150, 'subsample': 0.8}} \\
        & XGBoost & \texttt{{'colsample\_bytree': 1.0, 'learning\_rate': 0.1, 'max\_depth': 6, 'n\_estimators': 150, 'subsample': 0.8}} \\
        & LightGBM & \texttt{{'learning\_rate': 0.1, 'n\_estimators': 150, 'num\_leaves': 63, 'reg\_alpha': 0.1}} \\
        \midrule  
        \multirow{6}{*}{ C\&C } 
        & HRT-Boost & \texttt{{'learning\_rate': 0.1, 'max\_depth': 3, 'n\_estimators': 50, 'ridge\_alpha': 300}} \\
        & RF & \texttt{{'max\_depth': 7, 'max\_features': 'sqrt', 'min\_samples\_leaf': 4, 'n\_estimators': 150}} \\
        & AdaBoost & \texttt{{'learning\_rate': 0.1, 'loss': 'linear', 'n\_estimators': 50}} \\
        & Scikit-GBM & \texttt{{'learning\_rate': 0.1, 'max\_depth': 4, 'n\_estimators': 50, 'subsample': 0.8}} \\
        & XGBoost & \texttt{{'colsample\_bytree': 1.0, 'learning\_rate': 0.1, 'max\_depth': 3, 'n\_estimators': 50, 'subsample': 1.0}} \\
        & LightGBM & \texttt{{'learning\_rate': 0.1, 'n\_estimators': 50, 'num\_leaves': 63, 'reg\_alpha': 0.1}} \\
        \bottomrule
    \end{tabular}}
\end{table*}

\begin{table*}[htbp]
\centering
\caption{Optimal Hyperparameters for TabM and TabNet Across Various Datasets}
\label{tab:optimal_hyperparameters}\resizebox{\textwidth}{!}{
\begin{tabular}{lll}
\toprule
\textbf{Dataset} & \textbf{Model} & \textbf{ Hyperparameters } \\ 
\midrule
\multirow{2}{*}{Abalone} & TabM & \texttt{\{'d\_block': 128, 'dropout': 0.1, 'k': 16, 'learning\_rate': 0.002\}} \\
 & TabNet & \texttt{\{'learning\_rate': 0.02, 'n\_a': 16, 'n\_d': 8, 'n\_steps': 3\}} \\ \midrule
 \multirow{2}{*}{CPUact} & TabM & \texttt{\{'d\_block': 128, 'dropout': 0.0, 'k': 16, 'learning\_rate': 0.002\}} \\
 & TabNet & \texttt{\{'learning\_rate': 0.02, 'n\_a': 8, 'n\_d': 16, 'n\_steps': 4\}} \\ \midrule
\multirow{2}{*}{Ailerons} & TabM & \texttt{\{'d\_block': 128, 'dropout': 0.1, 'k': 8, 'learning\_rate': 0.002\}} \\
 & TabNet & \texttt{\{'learning\_rate': 0.02, 'n\_a': 16, 'n\_d': 16, 'n\_steps': 3\}} \\ \midrule
 \multirow{2}{*}{CTSlice} & TabM & \texttt{\{'d\_block': 128, 'dropout': 0.1, 'k': 16, 'learning\_rate': 0.001\}} \\
 & TabNet & \texttt{\{'learning\_rate': 0.02, 'n\_a': 16, 'n\_d': 16, 'n\_steps': 3\}} \\ \midrule
  \multirow{2}{*}{YearPred} & TabM & \texttt{\{'d\_block': 128, 'dropout': 0.1, 'k': 16, 'learning\_rate': 0.001\}} \\
 & TabNet & \texttt{\{'learning\_rate': 0.02, 'n\_a': 8, 'n\_d': 16, 'n\_steps': 3\}} \\ \midrule
 \multirow{2}{*}{Concrete} & TabM & \texttt{\{'d\_block': 128, 'dropout': 0.0, 'k': 8, 'learning\_rate': 0.002\}} \\
 & TabNet & \texttt{\{'learning\_rate': 0.02, 'n\_a': 16, 'n\_d': 16, 'n\_steps': 3\}} \\ \midrule
\multirow{2}{*}{Airfoil} & TabM & \texttt{\{'d\_block': 128, 'dropout': 0.0, 'k': 8, 'learning\_rate': 0.002\}} \\
 & TabNet & \texttt{\{'learning\_rate': 0.02, 'n\_a': 16, 'n\_d': 8, 'n\_steps': 4\}} \\ \midrule
\multirow{2}{*}{Fried} & TabM & \texttt{\{'d\_block': 128, 'dropout': 0.1, 'k': 8, 'learning\_rate': 0.002\}} \\
 & TabNet & \texttt{\{'learning\_rate': 0.01, 'n\_a': 16, 'n\_d': 8, 'n\_steps': 3\}} \\ \midrule

\multirow{2}{*}{D-Elevators} & TabM & \texttt{\{'d\_block': 128, 'dropout': 0.1, 'k': 8, 'learning\_rate': 0.002\}} \\
 & TabNet & \texttt{\{'learning\_rate': 0.02, 'n\_a': 8, 'n\_d': 16, 'n\_steps': 4\}} \\ \midrule
 \multirow{2}{*}{D-Ailerons} & TabM & \texttt{\{'d\_block': 128, 'dropout': 0.1, 'k': 8, 'learning\_rate': 0.002\}} \\
 & TabNet & \texttt{\{'learning\_rate': 0.02, 'n\_a': 8, 'n\_d': 8, 'n\_steps': 3\}} \\ \midrule

\multirow{2}{*}{Kinematics} & TabM & \texttt{\{'d\_block': 128, 'dropout': 0.0, 'k': 8, 'learning\_rate': 0.002\}} \\
 & TabNet & \texttt{\{'learning\_rate': 0.02, 'n\_a': 16, 'n\_d': 16, 'n\_steps': 3\}} \\ \midrule

\multirow{2}{*}{C\&C} & TabM & \texttt{\{'d\_block': 128, 'dropout': 0.1, 'k': 16, 'learning\_rate': 0.001\}} \\
 & TabNet & \texttt{\{'learning\_rate': 0.02, 'n\_a': 16, 'n\_d': 8, 'n\_steps': 3\}} \\ 
\bottomrule
\addlinespace[1ex]
\multicolumn{3}{l}{\footnotesize \textbf{Common Parameters:}} \\
\multicolumn{3}{l}{\footnotesize \textbullet \textbf{TabM}: \texttt{ 'batch\_size': 256, 'max\_epochs': 50, 'n\_blocks': 2, 'weight\_decay': 0.0003}} \\
\multicolumn{3}{l}{\footnotesize \textbullet \textbf{TabNet}: \texttt{'max\_epochs': 50, 'patience': 10}} \\
\end{tabular}}
\end{table*}

\clearpage
\newpage

\section{Estimation of model size and FLOPs}
\label{appendix_FLOPs}

We reported the single-sample inference FLOPs, and training FLOPs for each fitted model. All estimates were computed after preprocessing, using the transformed feature dimension. The reported FLOPs are analytic estimates derived from the fitted model structure, rather than hardware-profiler measurements.

HRT inference FLOPs were computed along a single root-to-leaf path. A dot product of length $p$ was counted as $p$ multiplications and $p-1$ additions, and each internal split also incurred one comparison. The reported HRT inference FLOPs were averaged over all root-to-leaf paths. For HRT-boost, parameters and inference FLOPs were summed over all HRT base learners, with additional scalar operations for scaling and accumulating the boosting outputs.

For tree-based baselines, including random forests, AdaBoost, gradient boosting, XGBoost, and LightGBM, inference FLOPs were estimated from the average number of split decisions visited by one sample, together with the scalar operations required to aggregate trees in ensemble models. For XGBoost and LightGBM, the estimates additionally accounted for missing-value or default-direction handling.

Training FLOPs for tree-based methods were estimated from their fitted structures and training settings. For sklearn-style trees, the main split-search cost at a node was taken to be proportional to
\[
    n_v d_v \log_2(n_v + 1),
\]
where $n_v$ is the number of samples reaching node $v$ and $d_v$ is the effective number of candidate features considered at that node. Additional terms were included for sample routing, leaf-value estimation, residual or gradient updates, sample-weight updates, and ensemble aggregation when applicable. For histogram-based boosted trees, including XGBoost and LightGBM, the estimates included gradient and Hessian computation, histogram construction, bin-wise split scanning, sample routing, and leaf-weight updates.

Deep learning baselines include TabM and TabNet.  TabM inference FLOPs were estimated from the transformed feature dimension, the number of ensemble heads, the number of blocks, and the hidden width, including input projection, block-wise transformations, element-wise operations, output projection, and head aggregation. TabNet inference FLOPs were estimated from the transformed feature dimension, the number of decision steps, and the decision and attention widths, including feature transformations, attentive-transformer operations, feature masking, sparse normalization, and decision aggregation. Neural training FLOPs were estimated as three times the forward-pass cost over all training examples and epochs, plus an optimizer-update cost proportional to the number of trainable parameters and the number of minibatches.

\section{Dataset Details}

$\bullet$ \textbf{Abalone:} This dataset was designed to predict the age of abalone, given eight features including categorical and numerical measurements describing physical characteristics.

$\bullet$ \textbf{CPUact:} This dataset was designed to predict the portion of time that CPUs run in user mode, given numerical features describing system performance measurements.

$\bullet$ \textbf{Ailerons:} This dataset was designed to predict the aileron control command, given numerical features describing the status of the aircraft.

$\bullet$ \textbf{CTSlice:} This dataset was designed to predict the relative location of CT slices on the axial axis of the human body, given numerical features describing bone and air distribution patterns extracted from CT images.

$\bullet$ \textbf{YearPred:} This dataset was designed to predict the release year of a song, given numerical features describing audio characteristics.

$\bullet$ \textbf{Concrete:} This dataset was designed to predict the compressive strength of concrete, given eight quantitative mixture components and the age of the concrete as input features.

$\bullet$ \textbf{Airfoil:} This dataset was designed to predict the scaled sound pressure level of airfoils, given five numerical features describing aerodynamic and geometric properties.

$\bullet$ \textbf{Fried:} This dataset was designed to predict a continuous target variable, given ten numerical features generated independently and uniformly.

$\bullet$ \textbf{D-Elevators:} This dataset was designed to predict the variation of elevator control signals for an F16 aircraft, given six numerical features describing the aircraft's state.

$\bullet$ \textbf{D-Ailerons:} This dataset was designed to predict the variation of aileron control signals for an F16 aircraft, given numerical features describing the aircraft's state.

$\bullet$ \textbf{Kinematics:} This dataset was designed to predict the forward kinematics of an 8-link robot arm, given numerical features describing joint configurations. The variant used (8nm) is highly nonlinear and moderately noisy.

$\bullet$ \textbf{C\&C (Communities \& Crime):} This dataset combines census, law-enforcement, and crime statistics to predict community crime rates. It is high-dimensional and heterogeneous, serving as a standard benchmark for socio-demographic prediction.

\begin{table}[htbp]
\small
	\centering
        \caption{Dataset Details }
         \label{tab:data_details}  
         \begin{tabular}{lcccccccc} 
         \toprule
                Dataset  & features & Total sample number   & Source\\ 
         \midrule
            Abalone & 8 & 4177  &  UCI\footnotemark[1]\\
            CPUact  & 21 & 8192  & Delve\footnotemark[2]\\
            Ailerons  & 40 &13750 & LIACC\footnotemark[3] \\
            CTSlice  & 384 &53500 &  UCI\footnotemark[1]\\
            YearPred  & 90 & 515345 &  UCI\footnotemark[1]\\
                        Concrete  & 8 & 1030 &  UCI\footnotemark[1]\\
            Airfoil  & 5 & 1503  &  UCI\footnotemark[1]\\
            Fried & 10 & 40768 & LIACC\footnotemark[3]\\
            D-Elevators  & 6 & 9517 & LIACC\footnotemark[3]\\
           D-Ailerons  & 5 & 7129 & LIACC\footnotemark[3]\\
            Kinematics  & 8 & 8192 & LIACC\footnotemark[3]\\
            C\&C & 127 & 1994  &  UCI\footnotemark[1]\\
         \bottomrule
         \end{tabular}
\end{table}

\footnotetext[1]{\url{https://archive.ics.uci.edu/}}
\footnotetext[2]{\url{https://www.cs.toronto.edu/~delve/data/comp-activ/desc.html}}
\footnotetext[3]{\url{https://www.dcc.fc.up.pt/~ltorgo/Regression/DataSets.html}}

\end{document}